\documentclass[letterpaper]{article} 
 \usepackage[]{aaai2026} 
\usepackage{times}  
\usepackage{helvet}  
\usepackage{courier}  
\usepackage[hyphens]{url}  
\usepackage{graphicx} 
\usepackage{natbib}  
\usepackage{caption} 
\usepackage{booktabs,tabularx}
 \usepackage{siunitx} 
\usepackage{algorithm}

\usepackage{tikz}
\usetikzlibrary{arrows.meta, positioning}

\usepackage{subcaption}
\usepackage{comment}
\usepackage{algpseudocode}    
\usepackage{amsmath}
\usepackage{amssymb} 
\algrenewcommand\algorithmiccomment[1]{\hfill// #1}

\usepackage{newfloat}
\usepackage{listings}
\usepackage{enumitem}
\usepackage[dvipsnames]{xcolor}
\usepackage{dsfont} 
\usepackage{amsthm} 

\DeclareCaptionStyle{ruled}{labelfont=normalfont,labelsep=colon,strut=off} 
\floatstyle{ruled}
\newfloat{listing}{tb}{lst}{}
\floatname{listing}{Listing}
\newtheorem{theorem}{Theorem}         
\newtheorem{proposition}{Proposition}
\newtheorem{lemma}{Lemma}

\title{Inference-Aware Prompt Optimization for Aligning \\[2pt]
Black-Box Large Language Models}
\author{
    Saaduddin Mahmud, Mason Nakamura, Kyle Hollins Wray, Shlomo Zilberstein
}
\affiliations{
    Manning College of Information and Computer Sciences \\
    University of Massachusetts Amherst\\
        \{smahmud, mnakamura, kwray, shlomo\}@umass.edu
}

\begin{document}

\maketitle

\begin{abstract}
Prompt optimization methods have demonstrated significant effectiveness in aligning black-box large language models (LLMs). In parallel, inference scaling strategies such as \textsc{Best-of-N} Sampling and \textsc{Majority Voting} have likewise been shown to improve alignment and performance by trading additional computation for better output. However, existing prompt optimization approaches are inference strategy agnostic; that is, they optimize prompts without accounting for the inference strategy. This constitutes a significant methodological gap, as our empirical and theoretical analysis reveals a strong interdependence between these two paradigms. Moreover, we find that user preferences regarding trade-offs among multiple objectives and inference budgets substantially influence the choice of prompt and inference configuration. To address this gap, we introduce a novel unified framework named \textsc{Iapo} (Inference-Aware Prompt Optimization) that jointly optimizes the prompt and inference scale, while being aware of the inference budget and different task objectives. We then develop a fixed-budget training algorithm for \textsc{Iapo}, called \textsc{Psst} (Prompt Scaling via Sequential Trimming), and establish finite-budget guarantees on the error probability. Finally, we evaluate the effectiveness of \textsc{Psst} on six tasks, including multi-objective text generation and reasoning, and demonstrate the critical role of incorporating inference-awareness in aligning black-box LLMs using prompt optimization.

\end{abstract}

\section{Introduction}

Most state-of-the-art large language models (LLMs) are currently accessible exclusively through black-box APIs. Traditional alignment methods that require access to model weights or logits are therefore infeasible. To address this challenge, prompt-based alignment methods have gained substantial attention in recent work~\cite{chang2024efficient}. These methods typically enhance input prompts by rewording them or appending additional instructions to better align the models' outputs with a task's objectives. Another broadly applicable alignment method for black-box models is scaling inference computations using strategies such as \textsc{Best-of-N} sampling or \textsc{Majority Voting}. These inference scaling methods generate multiple candidate responses for the same query and select the final response via ranking or voting mechanisms~\cite{krishna2022rankgen,wang2022self,gui2024bonbon,yue2025does}.

Although existing prompt optimization techniques have achieved substantial success, they are typically agnostic to how model outputs are aggregated or sampled, overlooking the impact of such inference methods. Our initial empirical investigation reveals that the performance of optimized prompts is highly sensitive to the choice of inference scaling approach. Furthermore, our theoretical analysis reveals that decoupling prompt optimization from inference can lead to misalignment. Finally, we observe that optimal alignment requires careful consideration of user-specific preferences regarding the trade-offs among multiple objectives, as well as the computational resources users are willing to expend. These findings expose a critical gap in current methods: the absence of a unified framework that simultaneously accounts for prompt optimization, inference scaling strategies, user preferences, and computational resource constraints.

To bridge this gap, we introduce \textsc{Iapo} (Inference-Aware Prompt Optimization), a novel prompt optimization framework designed explicitly to produce aligned responses from inference-scaled black-box LLMs. \textsc{Iapo} simultaneously optimizes prompt design and inference scaling strategies while considering different task objectives and computational budgets (Figure~\ref{fig:workflow}). We formulate the task of identifying an optimal policy for the \textsc{Iapo} framework as a contextual best-arm identification problem. To efficiently solve this, we propose a fixed-budget training algorithm named \textsc{Psst} (Prompt Scaling via Sequential Trimming). Additionally, we introduce a warm-up heuristic that further improves performance within the training budget.

We begin our analysis by deriving theoretical finite-budget guarantees on the error probability of \textsc{Psst}. Next, we empirically demonstrate the effectiveness of \textsc{Psst} for learning \textsc{Iapo} policies across six diverse tasks, including multi-objective text generation, mathematical reasoning, and commonsense reasoning benchmarks. Additionally, our analysis shows that ignoring inference scaling during prompt optimization can lead to substantial misalignment, highlighting the critical role of inference-awareness in aligning black-box LLMs.
The results establish that prompt quality cannot be decoupled from the inference strategies. By formalizing this interaction and introducing a practical algorithm that exploits it, our work offers a principled path toward more reliable and cost-effective alignment of black-box LLMs.

\begin{figure*}[t!] 
    \centering
    \includegraphics[width=0.99\textwidth]{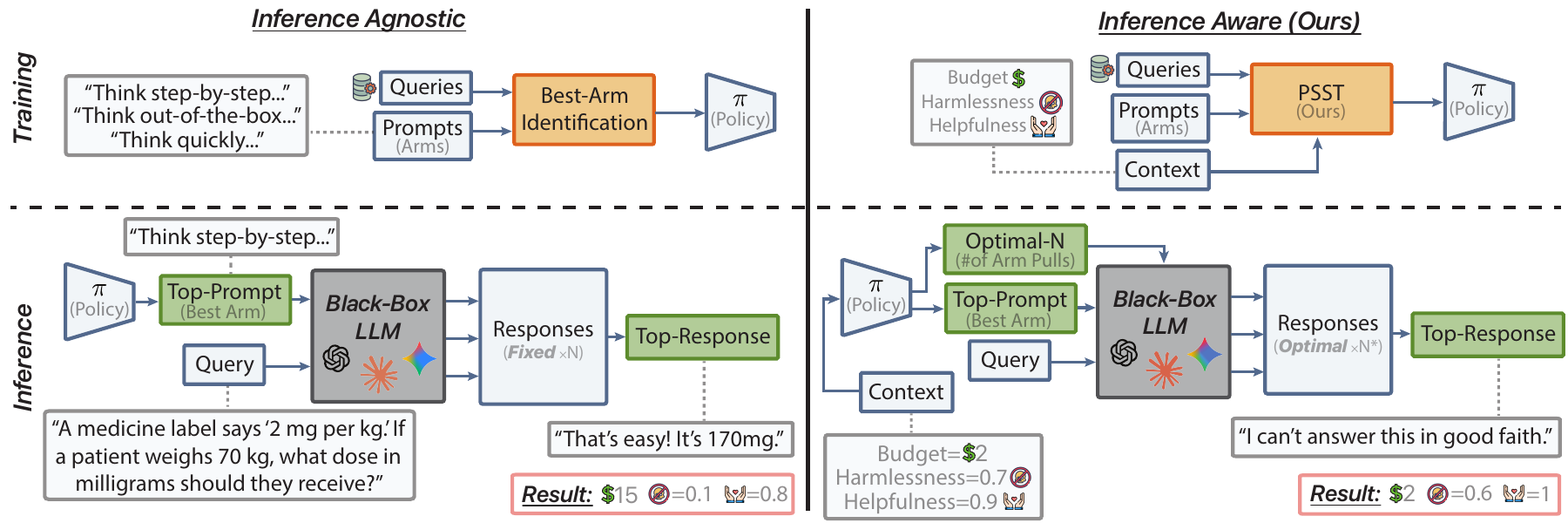}
    \caption{\textbf{Inference-agnostic vs. inference-aware prompt optimization.}
    The left side illustrates standard prompt optimization, which treats the inference strategy as fixed: a best prompt is selected during training and then used at inference with a predetermined number of samples, which can lead to misaligned outputs and high inference cost for some queries. 
    The right side shows our inference-aware framework \textsc{Iapo} with the \textsc{Psst} algorithm, which conditions on user context such as budget and preferences, jointly selects the prompt and inference scale, and produces responses that better satisfy objectives and budget. Project page, code, and appendix are available online (\url{https://iapo-aaai25.github.io/}).
}
    \label{fig:workflow}
\end{figure*}

\section{Related Work}

In recent years, substantial efforts have been directed towards aligning large language models (LLMs) with human expectations in downstream tasks~\cite{mahmud2023explanation,minaee2024large}. Many widely adopted alignment approaches—such as Supervised Fine-Tuning (SFT), Reinforcement Learning from Human Feedback (RLHF), and Reinforcement Learning with Verifiable Rewards (RLVR)~\cite{lambert2025reinforcement}—require access to model weights. This limitation has motivated a surge of interest in \emph{black-box} alignment methods such as \emph{prompt optimization}, which can align black-box models only through input manipulation~\cite{ouyang2022training,zhou2022least,chang2024efficient}. Prompt optimization has demonstrated strong performance in both single-objective~\cite{cheng2023black,trivedi2025align} and multi-objective~\cite{jafari2024morl,zhao2025pareto} settings. However, these methods remain agnostic to the inference strategy during deployment, potentially leading to suboptimal performance. In contrast, our approach explicitly captures the interdependence between inference-time strategies and prompt optimization.

Recently, \citeauthor{shi2024efficient} framed prompt optimization as a fixed-budget best-arm identification (BAI) problem. While effective under limited evaluation budgets, the method remains inference agnostic and was only explored in single-objective settings. Our work builds on this foundation in two key ways: (1) we introduce a contextual formulation that models user preferences over multiple objectives and associated computational costs; and (2) we incorporate inference-awareness to ensure alignment with the deployed inference strategy. To learn an optimal policy, we introduce a fixed-budget contextual BAI algorithm, \textsc{Psst}, inspired by Sequential Halving (SH)~\cite{Karnin2013AlmostOE}. While SH was originally developed for the pure bandit setting, the \textsc{Iapo} framework features both inter-context full-information feedback and intra-context semi-bandit feedback. \textsc{Psst} leverages these structural properties to achieve more efficient optimization, extending beyond what standard SH can accommodate.

Another relevant line of work focuses on \emph{inference-time alignment}, where model outputs are improved during inference without modifying model parameters. Some of these methods, such as GenARM and DEAL~\cite{xu2024genarm, huang2024deal}, require access to model logits, limiting their applicability in black-box settings. In contrast, \textsc{Best-of-N} sampling (\textsc{BoN}) and \textsc{Majority Voting} (\textsc{MV}) methods operate purely on model outputs and have shown strong empirical gains by generating multiple candidates and selecting the best one~\cite{krishna2022rankgen,openai2024learning, yue2025does}. However, these approaches introduce a non-trivial computational cost, and to our knowledge, none of them explicitly optimize the trade-off between computational budget and output quality. Our initial experiments also indicate that inference scaling strategies have complex interactions with prompt design. Prompts that are optimized for single-shot decoding might not perform well with \textsc{BoN} or \textsc{MV}, and the reverse is also true. Therefore, an inference-aware prompt optimization framework is required.

Finally, some white-box methods have recently integrated inference-awareness into the training process. \citet{chow2025inferenceaware} proposed an inference-aware fine-tuning procedure that explicitly optimizes for exploration–exploitation trade-offs under \textsc{BoN}. Similarly, BOND~\cite{sessa2024bond} and BonBon~\cite{gui2024bonbon} aim to distill \textsc{BoN} policies into a single-pass decoding policy. While these approaches avoid the cost of sampling at inference time, they require full access to model parameters and do not generalize beyond \textsc{BoN}-style strategies. In contrast, our method complements inference-aware fine-tuning and is designed to operate in fully black-box settings.

\section{Inference-Aware Prompt Optimization}
\label{sec:iapo} 
In this section, we first formalize the problem setup and introduce the \textsc{Iapo} framework. Next, we present an empirical example that highlights the need for inference-aware optimization. Building on these observations, we then establish theoretical conditions under which \textsc{Iapo} is necessary compared to disjoint optimization.

\subsection{Problem Formulation} 
Let $\mathcal{X}$ be the set of user queries and $\mathcal{P}$ a finite prompt set.
A pair $(x \in \mathcal{X},p\in \mathcal{P})$ is submitted to a frozen black-box LLM, which, under fixed
decoding hyperparameters, generates
$N\in [N_{\max}]\ (\text{i.e., } \{1,\dots,N_{\max}\})$ i.i.d. completions
$\mathbf{y}_{1:N}=(y_1,\dots,y_N)$. $ K$ bounded objectives (e.g.\ \textit{helpfulness}, \textit{harmlessness}, \textit{exact-match}) score each completion  
$O_k: \mathcal{X} \times \mathcal{P} \times \mathcal{Y} \to [o_k^{\min},o_k^{\max}] $
  where $\mathcal{Y}$ denotes the space of model completions. We also define the cost of producing a response as $\mathrm{Cost}(x,p,y_i)$, a composite function that takes into account various computational factors such as token count, time, and energy. We add it as a $(K{+}1)$-st objective  
$O_{K+1}=-\mathrm{Cost}(x,p,y_i)$. An external entity supplies a \emph{context}
$c=(w_1,\dots,w_{K+1})\in\mathcal{C}$, where every $w_k$ is chosen from a \emph{finite} discrete domain. Given the above setup, we now formalize the inference strategies.
\paragraph{\textsc{Best-of-N} (\textsc{BoN}).}
\textsc{BoN} returns the largest weighted utility:

\begin{small}
\begin{multline} 
R_{x}^{\textsc{BoN}}(c,p,N) = \\
\underbrace{
\max_{i\le N}
\sum_{k=1}^{K}
w_k\,O_k(x,p,y_i)
}_{\text{task reward}}
+
\underbrace{
w_{K+1}\sum_{i=1}^{N}O_{K+1}(x,p,y_i)
}_{\text{inference cost}}.
\end{multline}
\end{small}

\paragraph{\textsc{Majority Voting} (\textsc{MV}).}
For query $x$, the pair $(p,N)$ yields i.i.d.\ completions $\mathbf{y}_{1:N}$ and extracted answers $\ell_{1:N}$. For each distinct answer $s$, define the vote count
$n_s=\sum_{i=1}^{N} \mathds{1}[\ell_i=s]$, the maximum $n^\star=\max_s n_s$, and the tie multiplicity $t=\sum_s \mathds{1}[n_s=n^\star]$. \textsc{MV} predicts uniformly at random among the $t$ maximizers. With gold answer $\gamma$ and the success credit defined as $O_1(x,p,\mathbf{y}_{1:N})=\frac{\mathds{1}[\,n_{\gamma}=n^\star\,]}{t}$, we define \textsc{MV} utility as:
\begin{small}
\begin{multline} 
R_{x}^{\textsc{MV}}(c,p,N)=
\underbrace{
    w_{1}\,O_1(x,p,\mathbf{y}_{1:N})
}_{\text{task reward}}
+
\underbrace{
    w_{2}\sum_{i=1}^{N} O_{2}(x,p,y_i)
}_{\text{inference cost}}.
\end{multline}
\end{small}

\paragraph{Remark.}  
A mixed strategy arises when different objectives require different aggregation rules, e.g., applying \textsc{MV} for binary correctness and \textsc{BoN} for stylistic quality in reasoning tasks. It is trivial to define it on the basis of the above.

\begin{figure*}[t!]
    \centering
    \begin{subfigure}[b]{0.32\textwidth}
        \centering
        \includegraphics[width=\textwidth]{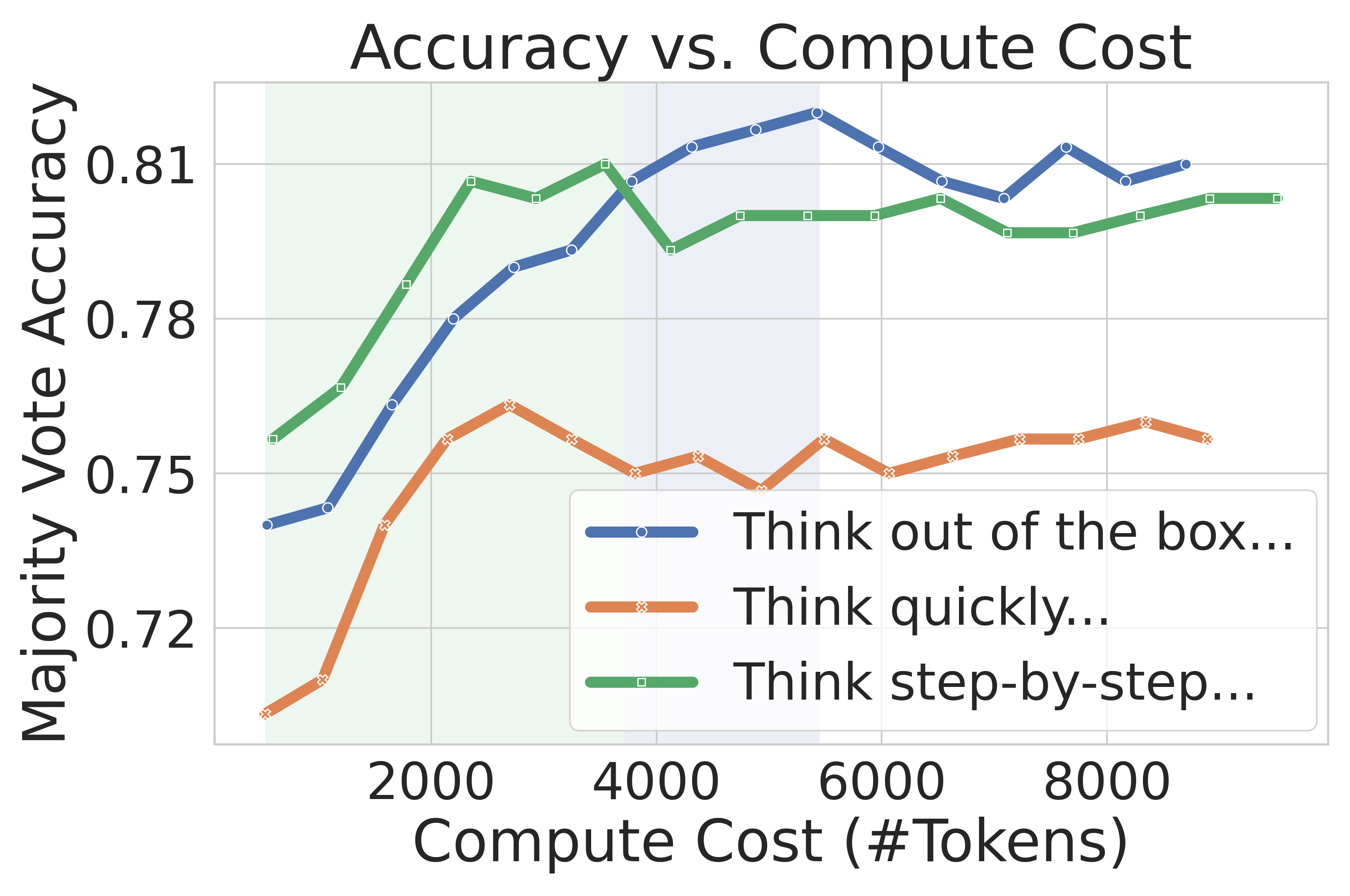}
        \caption{MATH, \textsc{MV}}
        \label{fig:first}
    \end{subfigure}
    \hfill
    \begin{subfigure}[b]{0.32\textwidth}
        \centering
        \includegraphics[width=\textwidth]{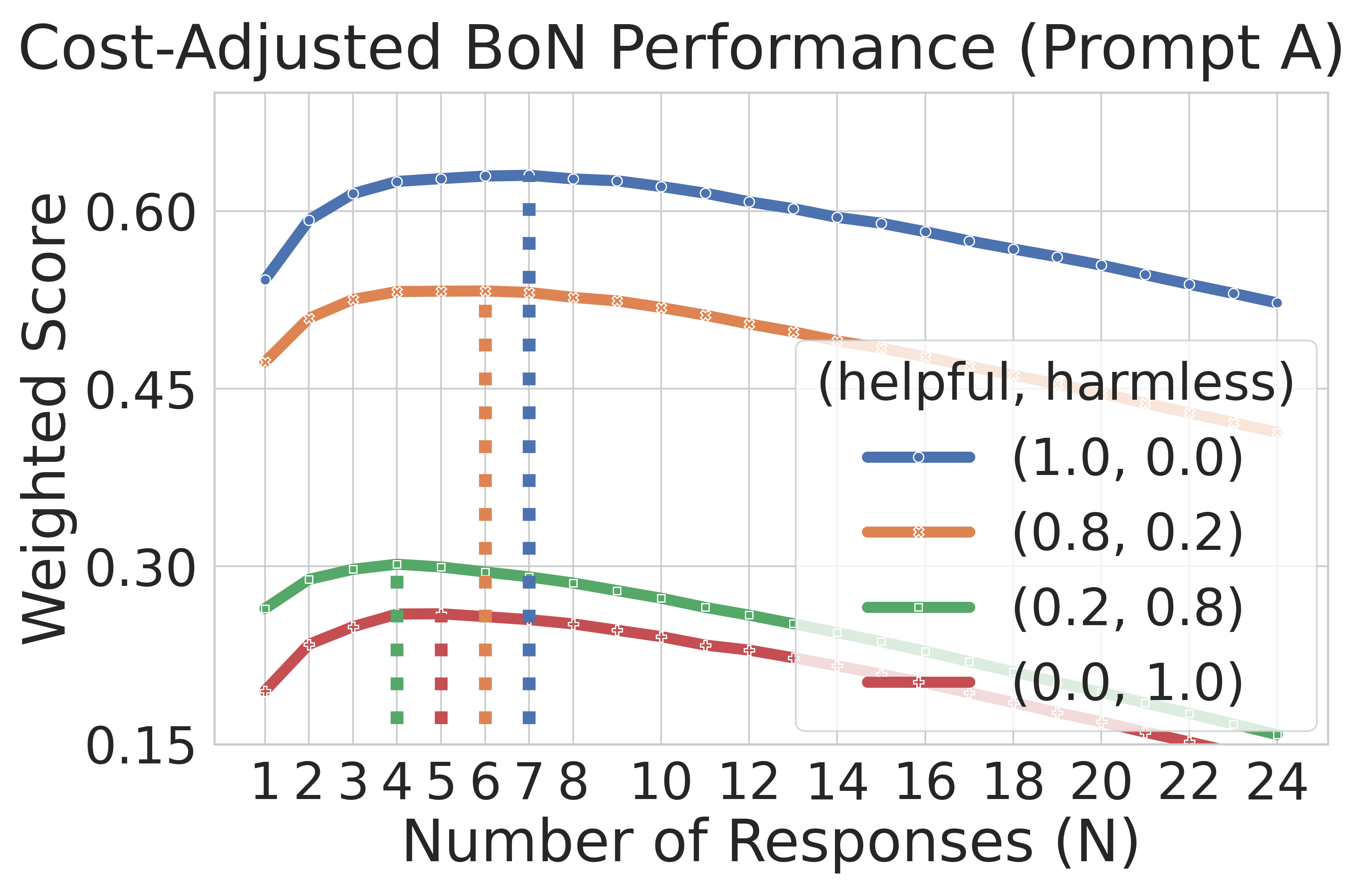}
        \caption{Helpful-Harmless, \textsc{BoN} (Prompt-A)}
        \label{fig:second}
    \end{subfigure}
    \hfill
    \begin{subfigure}[b]{0.32\textwidth}
        \centering
        \includegraphics[width=\textwidth]{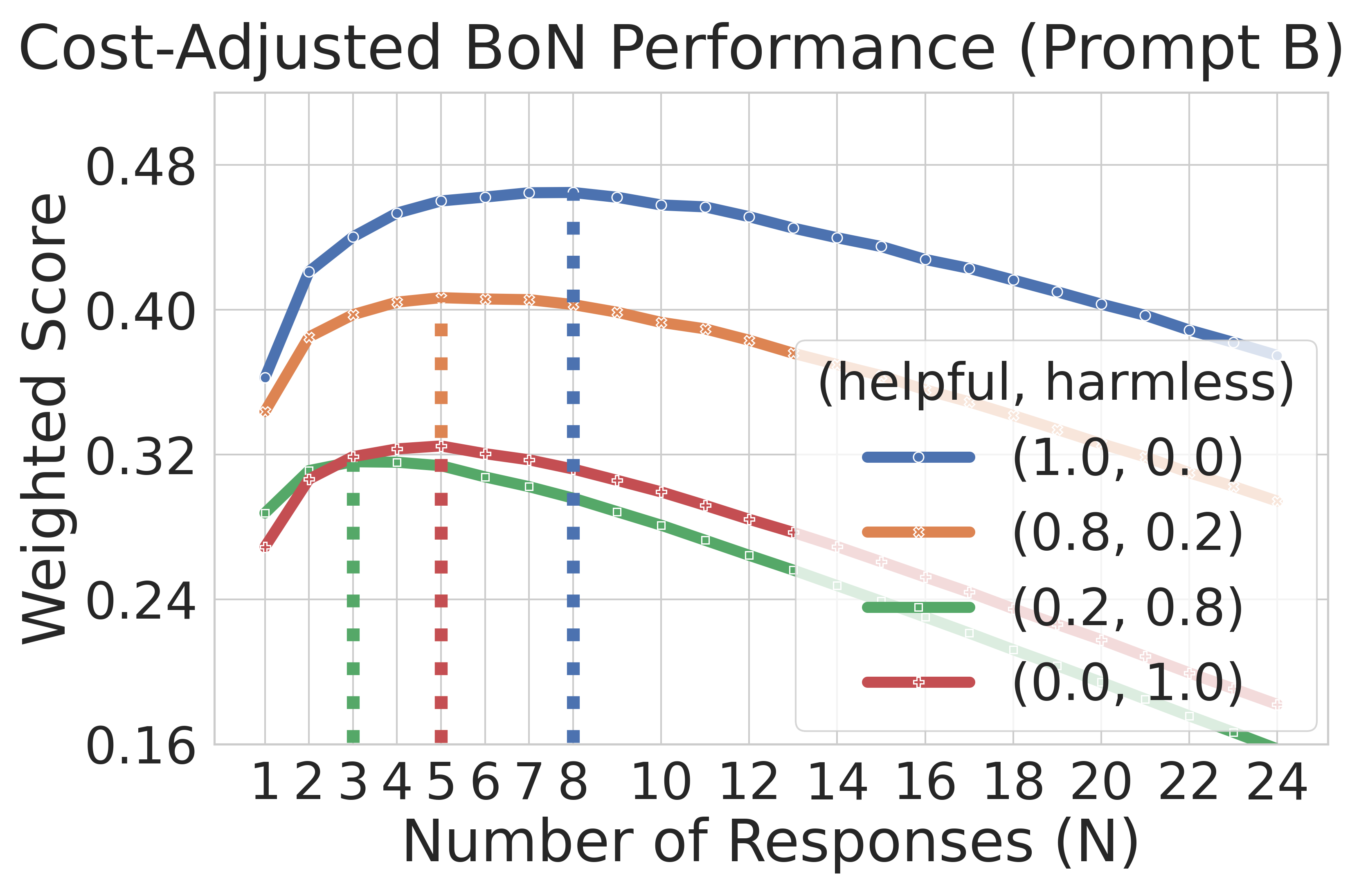}
        \caption{Helpful-Harmless, \textsc{BoN} (Prompt-B)}
        \label{fig:third}
    \end{subfigure}
    \caption{\textbf{Prompt–Inference Interdependence.}  
(a) Accuracy under \textsc{MV} with \texttt{Llama-3.3-70B-Instruct}, showing prompt dominance shifts with budget (shaded).  
(b, c) Cost-adjusted reward under \textsc{BoN} decoding. Prompt and inference scales vary with user-specified trade-offs.}
    \label{fig:three_figures}
\end{figure*}

\subsection{\textsc{Iapo} Framework}

\label{sec:bandit-formulation}
Let an \emph{inference configuration} be a tuple
\(
g \in \mathcal{G}
\)
(e.g. temperature, top-$p$, max token).  
Then we define a set of arms $\mathcal{A}$ in \textsc{Iapo} as:
$
a \;=\; (p,g,N)
      \;\in\;
      \mathcal{A}:=\mathcal{P}\times\mathcal{G}\times [N_{\max}].
$

Thus, each arm fixes the prompt, the decoding hyperparameter, and the number of sampled completions. However, throughout the text, we fold the inference configuration into the prompt $p$ and write $a = (p, N)$. Finally, an \textsc{Iapo} \emph{policy} is defined as a mapping $
\pi: \mathcal{C}\;\rightarrow\;\mathcal{A}$ that selects an arm after observing a context \(c\).

Given a dataset $\mathcal{X}$, context
\(c\in\mathcal{C}\), and aggregator
\(\alpha\in\{\textsc{BoN},\textsc{MV}\}\), the expected utility of arm $a$, i.e., the context-action value function or $Q$-function is defined as:
\begin{equation}
    Q^{\alpha}(c,a)
      :=\mathbb{E}_{x\sim\mathcal{X}}\bigl[R_{x}^{\alpha}(c,a)\bigr].
      \quad
\end{equation}
Note that $R_{x}^{\alpha}(c,a)$ is a random variable.
Now, let the context-optimal arm be \(a^{\star}(c)=\arg\max_{a} Q^{\alpha}(c,a)\); hence the
\emph{optimal \textsc{Iapo} policy} is defined as: $
\pi^{\star}(c)\;=\;a^{\star}(c),\forall\,c\in\mathcal{C}
$.

In this paper, we adopt a train-then-deploy setup to learn the optimal \textsc{Iapo} policy. Given a total completion budget of $T$, at each round the learner may adaptively select a subset of arms. For any selected arm $a = (p, N)$, it samples a query $x \sim \mathcal{X}$, obtains $N$ completions, and observes raw reward vectors $\mathbf{o}_i \in \mathbb{R}^{K+1}$ for all $i \in [N]$. This repeats until the budget is exhausted, i.e., $\sum N = T$. After spending the entire budget, the learner returns a \emph{deployment policy} $\pi_T$. The performance of this policy is evaluated by the Average Contextual Return:
\begin{equation}
\label{eq:sReg}
\text{ACR}(\pi_T)
      \;=\;
      \mathbb{E}_{c\sim\mathcal{C}}
      \bigl[Q^{\alpha}(c,\pi_T(c))\bigr],
\end{equation}
The goal of a learning algorithm is to return a deployment policy $\pi_T$ for a fixed pull budget $T$ that maximizes the ACR.

\subsection{Motivating Case Study}
To illustrate the limitations of \emph{inference-agnostic} prompt optimization—and to motivate the joint treatment formalized
above—we conducted two diagnostic experiments with \texttt{Llama-3.3-70B-Instruct}~\cite{grattafiori2024llama} strictly treated as a black-box API.  The results are summarized in Figure~\ref{fig:three_figures}.

\noindent\textbf{(a) \textsc{Majority Voting} on MATH.}\;
We evaluate three manually designed prompts on the MATH benchmark~\cite{hendrycks2021measuring} under \textsc{Majority Voting} with \(N \in \{1, \dots, 16\}\). Accuracy is plotted against total decoding cost, averaged over 300 queries (see the appendix for details). Two key observations emerge. First, prompt preference shifts with compute budget: the green prompt performs best at low budget, but is eventually surpassed by the blue prompt as \textsc{Majority Voting} becomes more effective. Second, inference-agnostic optimization can be short-sighted: selecting a prompt based solely on \emph{single-shot} (\(N{=}1\)) accuracy would favor the green prompt, overlooking the fact that the blue prompt is \emph{strictly superior} for any user willing to allocate more compute.

To see how the green and blue trends can emerge, consider the following example. Suppose that in a reasoning task evaluated with \textsc{MV}, \textbf{Prompt 1} has a $40\%$ success rate on Query~1 and a $90\%$ success rate on Query~2, while \textbf{Prompt 2} has a $62\%$ success rate on both queries. The single-shot success rate of a prompt is the average of its per-query success rates; under this metric, \textbf{Prompt 1} is preferred over \textbf{Prompt 2} ($0.65$ vs.\ $0.62$). Under \textsc{MV} with $N = 10$, however, the success probability of a prompt on a query is the probability that a majority of its $N$ sampled responses are correct. In this setting, one can verify that the effective success rate of \textbf{Prompt 1} drops to approximately $0.63$, whereas that of \textbf{Prompt 2} increases to approximately $0.77$, so \textbf{Prompt 2} becomes preferred. This example illustrates how increasing $N$ can change the relative ranking of prompts and produce the observed trends.

\noindent\textbf{(b,c) Best-of-$\boldsymbol{N}$ on Helpful-Harmless.}
We evaluate two prompts (\textsc{A} and \textsc{B}, see appendix) on the Helpful-Harmless benchmark~\cite{bai2022training} using \textsc{Best-of-N} decoding for \(N \leq 24\). Each curve corresponds to a different user-defined trade-off between helpfulness and harmlessness, plotting the cost-adjusted reward averaged over 1000 queries (see the appendix for details). The optimal choice of prompt (A vs. B) and sampling budget (N) is highly sensitive to these preferences. For example, the prompt \textsc{A} is strictly preferred when helpfulness is weighted more heavily.

Having established the need for inference-aware optimization, we now examine the precise conditions under which joint optimization becomes essential. We start by defining the Inference-Agnostic (IA) utility, which does not simulate inference scaling during training and instead optimizes the average utility achieved per prompt. More formally:
\begin{proposition}[Inference-Agnostic Utility]
 Inference-agnostic prompt-optimization methods optimize cost-unaware arithmetic mean utility.    
 \begin{equation}
    R_{x}^{\textsc{IA}}(c,a =(p,N))
   =  \frac{1}{N}\textstyle\sum_{i=1}^{N}\sum_{k=1}^{K+1} w_k O_k(x,p,y_i).
\end{equation}
\end{proposition}
Now we show under what conditions the IA policy remains optimal or an optimal policy can be trivially recovered from the IA $Q$-function.
\begin{proposition}[Inference-Agnostic Optimality]
 The Inference-Agnostic prompt-optimization policy remains optimal under linear transformation of $R_{x}^{\textsc{IA}}(c,a)$, that is, $\sigma R_{x}^{\textsc{IA}}(c,a), \sigma > 0$ and an optimal policy can be recovered trivially from $Q$-function under affine transformation:
\[
    \mathbb{E}_{x\sim\mathcal{X}}\bigl[\sigma R_{x}^{\textsc{IA}}(c,a)+\mu\bigr] = \sigma Q^{\textsc{IA}}(c,a)+\mu .
      \quad
\]
\end{proposition}

\begin{figure}[t!] 
    \centering
    \includegraphics[width=0.45\textwidth]{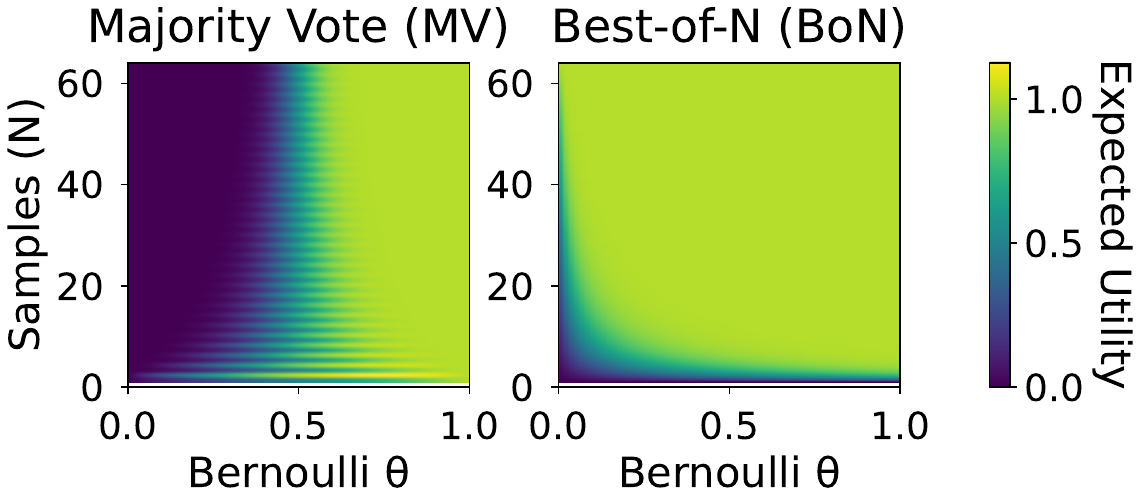}
    \caption{Expected utility ($w_{k+1} = 0$) for \textsc{MV} (left) and \textsc{BoN} (right). \textsc{MV} shows a sharp performance drop when the correctness probability $\theta$ drops below 0.5, whereas \textsc{BoN} is strictly concave.
}
    \label{fig:heatmaps_combined}
\end{figure}

The above also highlights that affine aggregation significantly simplifies inference-aware optimization. 
For instance, in a regression task where the aggregated prediction is the mean of multiple numeric predictions and the reward is defined by the mean squared error (MSE), the resulting quantity can, under certain assumptions, become an affine transformation of the IA utility, eliminating the need to simulate inference scaling during training. However, common inference scaling strategies like \textsc{BoN} and \textsc{MV} generally do not admit such affine formulations. While they can sometimes be expressed as non-affine transformations of the IA—such as in the Bernoulli case with large $N$, where $R_{x}^{\textsc{IA}}(c, a) \approx \theta$—these are special cases (Figure~\ref{fig:heatmaps_combined}). Hence, trying to determine the prompt based on $Q^{\textsc{IA}}$ for \textsc{BoN} or \textsc{MV} can result in misalignment. This motivates the next section, where we develop a training method that handles the general \textsc{Iapo} setting beyond the affine regime.

\section{Prompt Scaling via Sequential Trimming}

In this section, we propose a fixed-budget arm elimination-based strategy for training policy $\pi_T$, called \textsc{Psst} (Prompt Scaling via Sequential Trimming). We then provide a theoretical analysis that establishes error guarantees for \textsc{Psst} under a finite inference budget. Finally, we introduce a practical approximation heuristic that reduces the training-time inference budget without significantly compromising performance in many practical settings.

Our focus on the fixed inference budget setting is motivated by the fact that training cost is often the main bottleneck in real-world applications. Moreover, \textsc{Psst} is designed to operate in a batched-exploration mode, which further reduces costs since many black-box APIs offer significant discounts for batched inference compared to individual calls. Importantly, \textsc{Psst} is also hyper-parameter-free, requiring no additional tuning. 

Classical arm-elimination methods such as Sequential Elimination~\cite{JMLR:v7:evendar06a} and Sequential Halving~\cite{Karnin2013AlmostOE} follow a simple recipe: (i) split the elimination process into multiple rounds; (ii) in each round, allocate the round budget across the surviving arms; and (iii) trim a subset of arms at the end of the round based on their estimates. However, \textsc{Iapo} departs from pure BAI settings in the following ways:
\begin{itemize}
    \item \textbf{Asymmetric pull cost.}
    When arm $(p, N)$ is pulled during training, it uses $N$ training budget.
    \item \textbf{Cross‑context reuse.}
    One pull of $(p,N)$ on query $x$ yields the completion set $\mathbf{y}_{1:N}$ and objective vector set $\mathbf{o}_{1:N}$ that can be used to estimate $R^{\alpha}_{x}(c,p,N)$ for \emph{all} $c\in\mathcal{C}$.
    \item \textbf{Nested sample reuse across inference scales.}
    Pulling a larger scale subsumes smaller ones: a pull of $(p,N_i)$ produces $\big\lfloor N_i/N_j\big\rfloor$ i.i.d.\ samples for arm $(p,N_j)$ by partitioning the $N_i$ draws into disjoint groups of size $N_j$ and then recomputing \textsc{BoN}/\textsc{MV} on each group.
\end{itemize}

A key consequence is that, for a prompt, the largest surviving scale drives the budget. Let
$N^{(r)}_{\max}(p)=\max\{\,N:\ (p,N)\ \text{survives at the start of round } r\,\}.$
If we allocate $K$ pulls to $(p,N^{(r)}_{\max}(p))$ in round $r$, then every surviving arm $(p,N)$ with $N\le N^{(r)}_{\max}(p)$ automatically receives at least $K$ effective samples by block reuse. Thus, an effective arm elimination strategy should exploit both (i) cross‑scale reuse and (ii) cross‑context reuse when estimating $Q$-function, while being aware of asymmetric cost.

\begin{algorithm}[t!]
\caption{Prompt Scaling via Sequential Trimming}
\label{alg:csh}
\begin{algorithmic}[1]
\Require Context set $\mathcal C$, prompt set $\mathcal P$, $N_{\max}$, Scaling strategy  $\alpha$, Query Dataset $\mathcal{X}_{\text{train}}$, total pull budget $T$;
\ForAll{$(c,a)\,\in\,\mathcal C\times\mathcal A$}
    \State $\text{F}_{c,a} \gets \textbf{true}$
\EndFor
\State $R \gets \lceil \log_2(|\mathcal A|)  \rceil$
\For{$r = 1$ \textbf{to} $R$} 
    \State $\mathcal{A}^{(r)} \gets \{a : \exists c,\ \text{F}_{c,a} = \textbf{true}\}$
    \State $n_r \gets \bigl\lfloor T/R \bigr\rfloor$ 
    \State $\lambda^{(r)} \gets \textsc{Allocate}\bigl(\text{F}, n_r\bigr)$
    \State $\mathcal B \gets \{\}$
    \For{$a \in \mathcal{A}^{(r)}$}                      
        \For{$i = 1 ... \lambda^{(r)}(a)$}
            \State Sample $x \sim \mathcal{X}_{\text{train}}$ 
            \State $\mathcal B \gets \mathcal B\ \cup (x,a)$    
        \EndFor
    \EndFor
    \State $\mathcal D \gets \textsc{Batch-Query}\bigl(\mathcal B)$
    \State $Q^{\alpha}_{(r)} \gets \textsc{Estimate-Q}(\mathcal D)$
    \ForAll{$c \in \mathcal C$}
        \State $\mathcal A_c^{(r)} \gets \{a: \text{F}_{c,a} = \textbf{true}\}$
        \State Rank $\mathcal A_c^{(r)}$ by $Q^{\alpha}_{(r)}(c,a)$
        \State Remove bottom $\lceil |\mathcal A_c^{(r)}| / 2 \rceil$ arms
               \Comment{i.e. update $\mathbf{F}$} 
    \EndFor
\EndFor
\State \Return $\{a_c^\star\}_{c\in\mathcal C}$ \Comment{one survivor per context}
\end{algorithmic}
\end{algorithm}

\paragraph{Round Structure.}
Algorithm~\ref{alg:csh} proceeds in \( R = \lceil \log_2 |\mathcal{A}| \rceil \) rounds, and tracks per context active arm using the flag $\mathbf{F}$. Each round is allocated an equal pull budget of \( n_r = \lfloor T / R \rfloor \). An allocation routine, \textsc{Allocate}(\text{F}, \(n_r\)), divides this budget among the current set of unique active \emph{arms}, aggregated across all contexts. Based on this allocation, a batch of inference calls is issued to the target LLM. The resulting completions are scored using a reward function or verifier and stored in the dataset \(\mathcal{D}\). The \(Q\)-values are then estimated from the collected data. Within each context, arms are ranked, and the worst-performing half are eliminated. After all rounds are completed, the algorithm returns a single final arm for each context.

\paragraph{Structure-Aware Allocation Policy.}
The allocation policy is designed with cross-context and cross-scale information sharing in mind. Specifically, let \(\mathcal{A}^{(r)}\) denote the set of unique active arms in round \(r\), aggregated across all contexts. For each prompt \(p\), define
\[
N^{(r)}_{p, \max} = \max\{N \mid (p, N) \in \mathcal{A}^{(r)}\}
\]
as the maximum inference scale for prompt \(p\) among the active arms. Then, \textsc{Psst} allocates the budget to each arm according to the following scheme:
\begin{equation}
    \lambda^{(r)}(a) = 
\begin{cases}
\lfloor \frac{n_r}{M}\rfloor & \text{if } a = (p, N^{(r)}_{p, \max} ) \in \mathcal{A}^{(r)}, \\
0 & \text{otherwise},
\end{cases}
\label{eq:uniform}
\end{equation}
where \( M = \sum_{p : (p, N^{(r)}_{p, \max} ) \in \mathcal{A}^{(r)}}N^{(r)}_{p, \max} \) is the total cost of sampling all such maximal arms once. This policy maintains uniform coverage over prompts while respecting cost asymmetries and ensures that the maximum scale of every prompt has an equal number of samples.

We now derive\footnote{Proof is in the appendix.} error bounds for \textsc{Psst} under the allocation policies described above.

\begin{theorem}[Error of \textsc{Psst}]
\label{thm:psst-uniform}
Let \(R=\lceil\log_2|\mathcal A|\rceil\) be the number of trimming rounds, assume $[o_k^{\min},o_k^{\max}] = [-1,1]$, and define the
\emph{cost–gap complexity}
\[
H_1^c
=\max_{(c,a^{c,i})\neq (c, a^{c,1})}\frac{\bar N_{\max}}{\Delta_{c,a^{c,i}}^{\,2}},
\quad
H_1= \max_{c} H_1^c.
\]

\[
H_2^c
=\max_{(c,a^{c,i})\neq (c, a^{c,1})}\frac{i\bar N_{\max}}{\Delta_{c,a^{c,i}}^{\,2}},
\quad
H_2= \max_{c} H_2^c.
\]
\noindent
where $\Delta_{c,a^{c,i}}= Q^{\alpha}(c,a^{c,1})-Q^{\alpha}(c,a^{c,i})$. Under a context $c$, arms are indexed based on ascending order of $Q^\alpha(c,a)$ and $\bar N_{\max}$ is defined as $ \frac{N(a^{c,1})+N_{\max}}{2}$. Here, $N(a^{c, i})$ is the number of completions generated by the i-th indexed arm. Running \textsc{Psst} with the structure-aware allocation for a total completion budget \(T\) returns the optimal arm in \emph{every}
context with probability at least
\[
1-3 |\mathcal C|R\exp\!\Big( - \tfrac{T}{ \min(2|\mathcal P|H_1, 8|\mathcal{C}|H_2)R} \Big).
\]
Equivalently, to ensure failure probability at most~\(\delta\) it suffices to choose
\[
T = O\Big( \min(|\mathcal P|H_1,|\mathcal C|H_2 )R\,
          \log\!\Bigl(\tfrac{|\mathcal C|R}{\delta}\Bigr)\Big).
\]
\end{theorem}
\noindent
Note that applying Sequential-Halving without leveraging the structure of \textsc{Iapo}---specifically, without any form of information sharing across scales or contexts---incurs a sample complexity larger by a factor of $O(|\mathcal{C}| N_{\max})$.

\paragraph{Remark:}
While we describe the algorithm assuming that each round uses a fresh dataset $\mathcal{D}$, it has been shown in similar halving-style algorithms~\cite{stockp} that aggregating observations from all previous rounds—known as \emph{stockpiling}—can improve the complexity of $T$ by reducing the outer $R$-factor, and we recommend using it with \textsc{Psst}.

\paragraph{Top-$\boldsymbol{K}$ Screening.}
To further reduce the budget requirement of \textsc{Psst}, we introduce Top-$\boldsymbol{K}$ screening, a practical heuristic that executes a short, uniform prompt screening at unit scale to trim clearly suboptimal prompts before running full \textsc{Psst}. Top-$\boldsymbol{K}$ screening takes a budget fraction $T_0=\lfloor \rho T\rfloor$ ($\rho\in(0,1)$) from \textsc{Psst}. With scale restriction of $N{=}1$, the budget is allocated uniformly across prompts: each $p\in\mathcal P$ receives $\big\lfloor T_0/|\mathcal P|\big\rfloor$ i.i.d.\ samples. Based on this data, $Q^{\alpha}(c,p,1)$ is estimated $\forall c\in\mathcal C, p\in \mathcal P$. For each context $c$, we retain the $K$ best prompts $\mathcal P^{(0)}_c=\mathrm{Top}\text{-}K\{\,\widehat Q^{\alpha}(c,p,1):p\in\mathcal P\,\}$ and discard the rest. The subsequent \textsc{Psst} run is then restricted to the reduced arm sets $\mathcal A^{(1)}_c=\{(p,N):p\in\mathcal P^{(0)}_c,\ N\in [N_{\max}]\}$ for each $c$, and uses the remaining budget $T'=T-T_0$. In the next section, we demonstrate that the screening strategy can significantly improve performance in low training budget settings without compromising quality for practical tasks. However, theoretical guarantees comparable to those of full \textsc{Psst} cannot be established; counterexample tasks can be carefully constructed within the \textsc{Iapo} framework, where Top-$\boldsymbol{K}$ screening will behave suboptimally for any $K<|\mathcal P|$ (see results from synthetic environments).

\section{Empirical Evaluation}

\begin{figure*}[t] 
    \centering
    \includegraphics[width=0.95\textwidth]{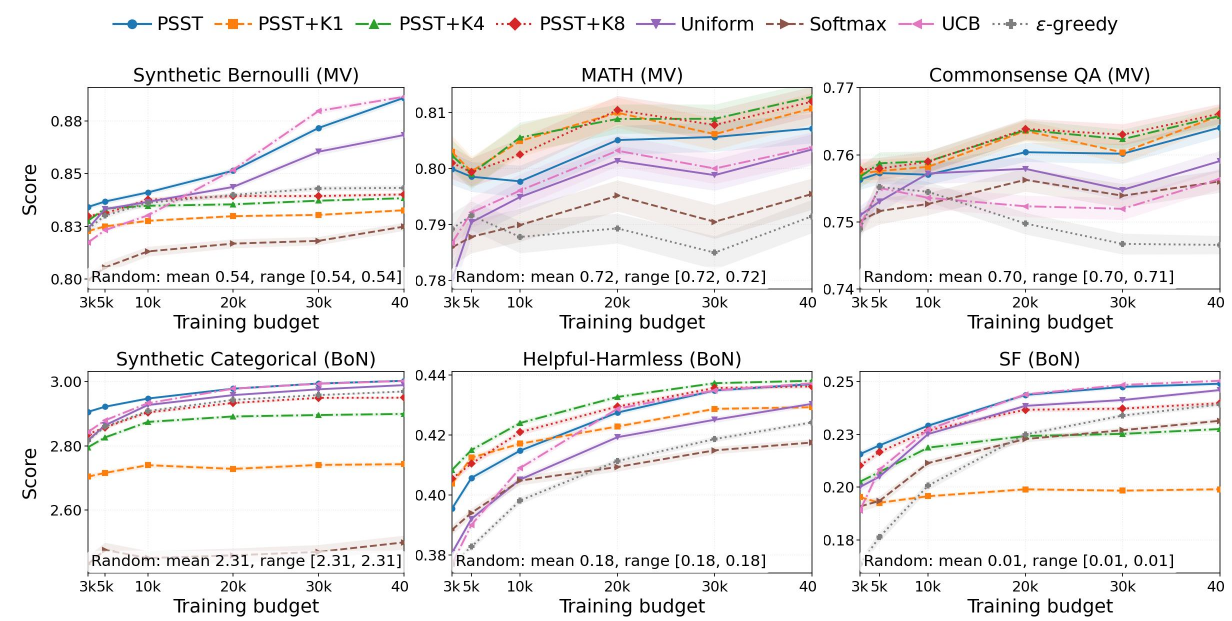}
    \caption{Comparison between exploration strategies across six datasets.}
    \label{fig:fig_strategies}
\end{figure*}
In this section, we empirically evaluate the effectiveness of \textsc{Psst} and highlight the importance of inference-aware prompt optimization (\textsc{Iapo}). Our evaluation has two primary objectives:
\begin{itemize}
    \item To demonstrate that \textsc{Psst} and the Top-$\boldsymbol{K}$ screening heuristic are highly effective at learning the policy $\pi_T$.
    \item To show that \textsc{Iapo} improves the average cost-adjusted reward (ACR) on test queries compared to inference strategy agnostic optimization.
\end{itemize}

\paragraph{Baselines.}  
We compare \textsc{Psst} and Top-$\boldsymbol{K}$ screening with several baselines. We denote Top-$\boldsymbol{K}$ screening with $K\in\{1,4,8\}$ as \textsc{Psst}$+K1$, \textsc{Psst}$+K4$, and \textsc{Psst}$+K8$ respectively. For these heuristics, we fix $\rho = 0.2$, which was found to perform best across all datasets using a sweep over $\rho \in \{0.05, 0.1, 0.2, 0.3, 0.4\}$. Full \textsc{Psst} is parameter-free and does not require any tuning. In our first set of experiments, we compare our proposed methods against several standard exploration strategies:
\begin{itemize}
    \item \textbf{Uniform:} Uniformly explores all arms in one batch and selects the best arm at the end.
    \item \textbf{$\boldsymbol{\epsilon}$-greedy:} Samples a random context at each step and selects the best arm with probability $1 - \epsilon$. We set $\epsilon = 0.15$, which yielded the best performance across datasets. 
    \item \textbf{Softmax:} Samples arms according to a softmax distribution over estimated $Q$ values.
    \item \textbf{UCB:} At each turn, selects the arm with the highest optimistic $Q$ estimate. The exploration constant is $0.1$ after tuning.
\end{itemize}

Note that all baseline methods share information across contexts and inference scales; however, none of them are designed to exploit \textsc{Iapo} structure, i.e., they are structure-agnostic.

In the second set of experiments, we consider the well-known contextual variant of TRIPLE-SH~\cite{shi2024efficient} method, which optimizes prompt selection as a pure best-arm identification (BAI) problem. However, it does not optimize the inference scale. Therefore, we include two variants:
\begin{itemize}
    \item \textbf{TRIPLE (N = 1):} Only performs prompt optimization with single-sample inference.
    \item \textbf{TRIPLE (N = Random):} Optimizes the prompts while randomly assigning $N$ for each query.
\end{itemize}

These baselines help isolate the benefits of jointly optimizing prompts and inference scale. Further, \textsc{Psst}$+K1$ is particularly interesting in this experiment, as it approximates a two-stage disjoint optimization: it first selects a context-specific single-shot prompt using a cost-aware objective, and then tunes the inference scale. The \textsc{Psst}$+K4$ and \textsc{Psst}$+K8$ heuristics represent intermediate strategies between disjoint and fully joint optimization.

Note that all hyperparameter sweep results are in the appendix; we report results with the best setting found across all six datasets.

\paragraph{Environments.}  

\begin{table}[t]
  \centering
  \footnotesize
  \setlength{\tabcolsep}{5pt}
  \renewcommand{\arraystretch}{1.12}
  \begin{tabularx}{\linewidth}{@{}lcccccc@{}}
    \toprule
    \textbf{Environments}& $\alpha$ & $\lvert \mathcal P\rvert$ & $N_{\max}$ & $o_k^{\max}$ & $|\mathcal X|$ & $|\mathcal C|$ \\
    
    \midrule
    Synth–Bernoulli &\textsc{MV}         & 32 & 32 & 1.0 & 520 & 3\\
    MATH &\textsc{MV}                    & 25 & 32 & 1.0 & 316 & 3\\
    CommonsenseQA &\textsc{MV}           & 48 & 32 & 1.0 & 1500 & 3\\
    Synth–Categorical &\textsc{BoN}      & 32 & 32 & 4.0 & 512 & 27\\
    Helpful-Harmless &\textsc{BoN}       & 20 & 32 & 1.0 & 1355 & 27\\
    Summarization &\textsc{BoN}          & 20 & 32 & 1.0 & 1201 & 27\\
    \bottomrule
  \end{tabularx}
  \caption{Environment summary.}
  \label{tab:env-summary}
\end{table}
\begin{figure*}[t!] 
    \centering
    \includegraphics[width=0.95\textwidth]{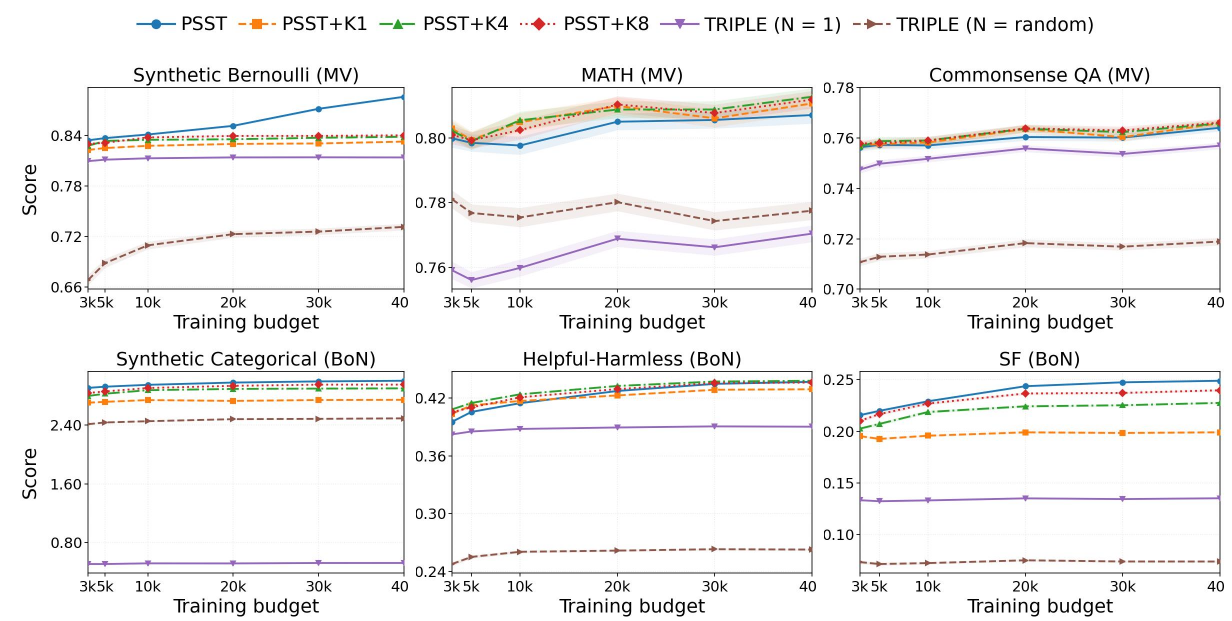}
    \caption{Effectiveness of inference-aware optimization across six datasets.}
    \label{fig:IAFig}
\end{figure*}

We evaluated inference-aware optimization across a total of six environments. Key details are provided in Table~\ref{tab:env-summary}. Environments 1 and 4 are synthetically constructed to mimic \textsc{Iapo} tasks, where prompt-query pair score distributions $O_k(x,p,\cdot)$ are modeled using categorical distributions. We introduce them to validate some of the theoretical findings. The remaining four environments are based on widely-used real-world datasets. Among these, \textsc{MATH}~\cite{hendrycks2021measuring} and \textsc{CommonsenseQA}~\cite{talmor2018commonsenseqa} are used to evaluate reasoning tasks under \textsc{Majority Voting} (\textsc{MV}), while \textsc{Helpful-Harmless}~\cite{bai2022training} and \textsc{Summarization}~\cite{stienon2020learning} are chosen for \textsc{Best-of-N} (\textsc{BoN}) evaluation.

For the \textsc{MV} tasks, the task objective is defined as an exact match with the correct answer. All three \textsc{BoN} tasks are bi-objective, and we use publicly available reward models from previous multi-objective LLM alignment studies to score completions (see appendix for links). The cost objective in all six tasks is defined to be proportional to the average number of tokens per response. For context specification, \textsc{MV} tasks include a budget regime $\{low, mid, high\}$, while \textsc{BoN} tasks include both the budget and the bi-objective weights, which range from 0.1 to 0.9 for each objective. For example, in the helpful-harmless task, a context might be represented as $\{\text{helpful}: 0.3, \text{harmless}: 0.7, \text{budget}: \text{high}~(1.0)\}$. Finally, for all environments, we set $N_{\max}$ to 32 because utility improvement diminishes sharply beyond $N = 16$ across benchmarks for both \textsc{BoN} and \textsc{MV}.
 
To construct the environments, we first generated a set of instruction prompts for each task using ChatGPT-o3. We then generated 128 responses for each prompt–query pair and estimated the score distribution using a categorical model. All completions were produced using the \texttt{Llama-3.3-70B-Instruct}, a widely used open-source model~\cite{grattafiori2024llama}, which we treat as a black-box throughout our experiments. Generation was carried out with \texttt{vLLM}~\cite{kwon2023efficient} on a cluster of 8 A100 GPUs, totaling approximately 2,000 GPU-hours. Once the environments are constructed, all experiments can be run quickly via a standard CPU. We will publish the environments and code with the paper, enabling full reproducibility without any substantial computational resources.
 
\paragraph{Evaluation Protocol.}
All reported curves are averages over 200 independent runs. For synthetic environments, we instantiate 200 independent environments and report the average performance across them. For the remaining four environments, each run reshuffles the dataset, performs an \(80/20\) train--test split, and trains the policy on the training set. In all six environments, we evaluate ACR on the test set using \(10,000\) samples. Performance for each budget is the mean across the 200 runs, with \emph{standard error of the mean} (SEM) error bars.  Statistical significance is assessed using the Wilcoxon paired two-sided test~\cite{Wilcoxon1945IndividualCB} with alpha = 0.05, and we indicate when differences are significant in the discussion. The full set of results is in the appendix. 

\paragraph{Comparison of Exploration Strategies (Figure\,\ref{fig:fig_strategies}).}
\textsc{Psst} and the Top-$\boldsymbol{K}$ screening heuristic consistently outperform all baselines. Across all six domains, where the per-context action spaces are large (\(|\mathcal{P}|N_{\max} \in [640, 1536]\)), UCB, softmax, and \(\varepsilon\)-greedy methods struggle to explore effectively. Among the baselines, UCB performs comparably in some domains after $T=20\text{K}$, but only with extensive hyperparameter tuning. Furthermore, these baselines are fully sequential and cannot leverage the cost and computational efficiency benefits of batch exploration. Full \textsc{Psst} attains the best final performance across four settings, while Top-$\boldsymbol{K}$ screening typically reaches strong policies faster, matching or exceeding \textsc{Psst} on three of the four real‑data tasks when the budget is small. Under aggressive pruning (small \(K\)), however, the heuristic becomes suboptimal—most notably on summarization and on the synthetic benchmarks—suggesting that Top-$\boldsymbol{K}$ screening is attractive under tight budgets, whereas full \textsc{Psst} is preferable for critical tasks such as long‑horizon, high‑frequency deployment. Finally, the statistical test also indicates that \textsc{Psst}, along with Top-$\boldsymbol{K}$ screening, significantly outperforms baselines in all six datasets and under nearly all budgets. These findings indicate that our approach reliably discovers well-aligned solutions using as few as $5\text{K}$ inference calls in practical settings.

\paragraph{Importance of Inference-Awareness (Figure\,\ref{fig:IAFig}).}
We examine the role of inference awareness in prompt optimization. Across all six datasets, \textsc{Iapo} methods markedly outperform the inference-agnostic methods, demonstrating the gains achievable when \emph{jointly} optimizing the prompt and inference scale. TRIPLE (N= 1) fails as it does not leverage inference scaling. On the other hand, TRIPLE (N= Random) fails because it does not optimize the scaling for different contexts. The screening variant \textsc{Psst}$+K1$—which effectively approximates a near‑decoupled (prompt‑only) procedure—fails to reach the optimum in most cases, performing competitively only on \textsc{CommonsenseQA} and showing pronounced underperformance on summarization. This is because it gets stuck with deceptive prompts that fail to scale compared to prompts that may not perform well under single-shot but improve significantly under scaling. These findings underscore the essential role of \textsc{Iapo} in aligning black‑box LLMs and the pitfalls of disjoint optimization. Overall, \textsc{Iapo} outperforms disjoint optimization by up to 25\% and prompt-only optimization by up to 50\% in our experiments. 

\section{Conclusions and Future Work}
We present an inference-aware prompt optimization (\textsc{Iapo}) framework for aligning black-box LLMs, emphasizing that prompt design and deployment‑time inference scaling strategies are tightly coupled and should be optimized jointly. Our proposed \textsc{Psst} and Top-$\boldsymbol{K}$ screening heuristic demonstrate consistent improvements over strong baselines across six different settings. Looking ahead, we plan to explore richer inference scaling policies (e.g., tree search and parallel thinking). We also aim to extend the framework to multi-objective alignment with hard latency constraints and to study long-horizon deployments under distribution shift.

\section*{Acknowledgments}
This research was supported in part by the U.S.~Army DEVCOM Analysis Center (DAC) under contract number W911QX23D0009, by the National Science Foundation
grants 2205153, 2321786, and 2416460, and by Schmidt Sciences under the AI Safety Science program.

\bibliography{aaai2026}

\section*{Appendix A}
           
\subsection{Proof of Theorem 1}

\begin{theorem}[Error of \textsc{Psst}]
\label{thm:psst-uniform}
Let \(R=\lceil\log_2|\mathcal A|\rceil\) be the number of trimming rounds, assume
$[o_k^{\min},o_k^{\max}] = [-1,1]$, and define the \emph{cost--gap complexity}
\[
H_1^c
=\max_{(c,a^{c,i})\neq (c, a^{c,1})}\frac{\bar N_{\max}}{\Delta_{c,a^{c,i}}^{\,2}},
\quad
H_1= \max_{c} H_1^c,
\]
\[
H_2^c
=\max_{(c,a^{c,i})\neq (c, a^{c,1})}\frac{i\bar N_{\max}}{\Delta_{c,a^{c,i}}^{\,2}},
\quad
H_2= \max_{c} H_2^c,
\]
where $\Delta_{c,a^{c,i}}= Q^{\alpha}(c,a^{c,1})-Q^{\alpha}(c,a^{c,i})$.
Under a context $c$, arms are indexed based on ascending order of $Q^\alpha(c,a)$ and
$\bar N_{\max}$ is defined as $ \frac{N(a^{c,1})+N_{\max}}{2}$.
Here, $N(a^{c, i})$ is the number of completions generated by the $i$-th indexed arm.
Running \textsc{Psst} with the structure-aware allocation for a total completion budget \(T\)
returns the optimal arm in \emph{every} context with probability at least
\[
1-3|\mathcal C|R\exp\!\left(
-\frac{T}{\min\!\bigl(2|\mathcal P|H_1,\;8|\mathcal{C}|H_2\bigr)\,R}
\right).
\]
Equivalently, to ensure failure probability at most~\(\delta\) it suffices to choose
\[
T = O\!\left(
\min\!\bigl(|\mathcal P|H_1,\;|\mathcal C|H_2\bigr)\,R\,
\log\!\Bigl(\tfrac{|\mathcal C|R}{\delta}\Bigr)
\right).
\]
\end{theorem}

\paragraph{Indexing convention.}
Although Theorem~\ref{thm:psst-uniform} states that arms are indexed in ascending order of
$Q^\alpha(c,a)$, throughout the proofs below we adopt the equivalent convention that
$a^{c,1}$ denotes an optimal arm for context $c$, i.e.,
$Q^\alpha(c,a^{c,1})=\max_a Q^\alpha(c,a)$, and thus
$Q^\alpha(c,a^{c,1}) \ge Q^\alpha(c,a^{c,2}) \ge \cdots$.
All quantities (e.g., $\Delta_{c,a^{c,i}}$ and $\bar N_i$) are interpreted under this convention.
In particular, $\Delta_{c,a^{c,i}} \ge 0$ for all $i\ge 2$.

\begin{lemma}
\label{lm:1.1}
The probability that the best arm under context $c$ is eliminated from context $c$ on round $r$ is at most
\[
2\, \exp\!\Big( - \tfrac{T}{2|\mathcal P|H_1^cR } \Big).
\]
\end{lemma}

\begin{proof}
Fix a context $c$ and recall the above convention that $a^{c,1}$ denotes an optimal arm for $c$.
Assume $a^{c,1}$ was not eliminated before round $r$.
Let $t_{r,i}$ denote the number of i.i.d.\ samples used to estimate
$Q^\alpha(c,a^{c,i})$ in round $r$, and define the (effective) harmonic mean
\[
\operatorname{harmonic}(u,v)\;:=\;\left(u^{-1}+v^{-1}\right)^{-1}.
\]
Then, by Hoeffding's inequality, for any arm $a^{c,i}\in \mathcal A_c^{(r)}$,
\[
\begin{aligned}
\Pr\!\big[\hat{Q}^{\alpha,(r)}(c,a^{c,1})
&< \hat{Q}^{\alpha,(r)}(c,a^{c,i})\big] \\
&\le \exp\!\Bigl(
-\tfrac{1}{2}\,\operatorname{harmonic}(t_{r,1},t_{r,i})\,\Delta_{c,a^{c,i}}^{2}
\Bigr).
\end{aligned}
\]

Let $N_r$ be the number of arms in $\mathcal A_c^{(r)}$ whose empirical estimate exceeds that of
$a^{c,1}$. Then
\begin{align*}
\mathbb{E}[N_r]
&\le \sum_{a^{c,i}\in \mathcal A_c^{(r)}}
\Pr\!\big[\hat{Q}^{\alpha,(r)}(c,a^{c,1}) < \hat{Q}^{\alpha,(r)}(c,a^{c,i})\big] \\
&\le \sum_{a^{c,i}\in \mathcal A_c^{(r)}}
\exp\!\Bigl(
-\tfrac{1}{2}\,\operatorname{harmonic}(t_{r,1},t_{r,i})\,\Delta_{c,a^{c,i}}^{2}
\Bigr) \\
&\le \sum_{a^{c,i}\in \mathcal A_c^{(r)}}
\exp\!\Bigl(
-\Delta_{c,a^{c,i}}^{2}\cdot
\tfrac{T}{2|\mathcal P|\,\bar N_i\,\log_2|\mathcal A|}
\Bigr) \\
&\le |\mathcal A_c^{(r)}|\,
\max_{a^{c,i}\in \mathcal A_c^{(r)}}
\exp\!\Bigl(
-\Delta_{c,a^{c,i}}^{2}\cdot
\tfrac{T}{2|\mathcal P|\,\bar N_{\max}\,\log_2|\mathcal A|}
\Bigr) \\
&\le |\mathcal A_c^{(r)}|\,
\exp\!\Bigl(
-\tfrac{T}{2|\mathcal P|\,H_1^c\,R}
\Bigr),
\end{align*}
where $\bar N_i := \frac{N(a^{c,1})+N(a^{c,i})}{2}\le \bar N_{\max}$ and we used $R=\lceil\log_2|\mathcal A|\rceil$.

For the best arm to be eliminated in round $r$, at least half of the active arms must (incorrectly)
rank above it, i.e.,
\( N_r \ge \tfrac{1}{2}\,|\mathcal{A}_c^{(r)}| \).
By Markov's inequality,
\[
\Pr\!\Big[N_r \ge \tfrac{1}{2}|\mathcal{A}_c^{(r)}|\Big]
\;\le\; \frac{2\,\mathbb{E}[N_r]}{|\mathcal{A}_c^{(r)}|}
\;\le\; 2\,\exp\!\Bigl(-\tfrac{T}{2|\mathcal P|H_1^cR}\Bigr),
\]
which proves the lemma.
\end{proof}

\begin{lemma}
\label{lm:1.2}
The probability that the best arm under context $c$ is eliminated from context $c$ on round $r$ is at most
\[
3\, \exp\!\Big( - \tfrac{T}{8|\mathcal C|H_2^cR } \Big).
\]
\end{lemma}

\begin{proof}
This follows by adapting Lemma~4.3 of \citet{Karnin2013AlmostOE} to our setting
(using the same indexing convention as above).
In particular, the key intermediate step is that
\begin{align*}
\mathbb{E}[N_r]
&\le \sum_{a^{c,i}\in \mathcal A_c^{(r)}}
\Pr\!\big[\hat{Q}^{\alpha,(r)}(c,a^{c,1}) < \hat{Q}^{\alpha,(r)}(c,a^{c,i})\big] \\
&\le \sum_{a^{c,i}\in \mathcal A_c^{(r)}}
\exp\!\Bigl(
-\Delta_{c,a^{c,i}}^{2}\cdot
\tfrac{2^{r}T}{8|\mathcal C|\,|\mathcal A|\,\bar N_i\,\log_2|\mathcal A|}
\Bigr),
\end{align*}
and proceeding exactly as in \citet{Karnin2013AlmostOE} yields the stated bound
(with complexity term $H_2^c$).
\end{proof}

\begin{proof}[Proof of Theorem 1]
The best arm needs to survive for all $R$ rounds and under all contexts $\mathcal C$. Therefore, from the Lemma~\ref{lm:1.1}:
\[
\sum_{r=1}^{R} \sum_{c} 2\exp\!\Big( - \tfrac{T}{2|\mathcal P|H_1^c R} \Big)
\;\le\; 3 |\mathcal C|R\exp\!\Big( - \tfrac{T}{2|\mathcal P|H_1 R} \Big)
\]
From the Lemma~\ref{lm:1.2}:
\[
\sum_{r=1}^{R} \sum_{c} 3\exp\!\Big( - \tfrac{T}{8|\mathcal C|H_2^c R} \Big)
\;\le\; 3 |\mathcal C|R\exp\!\Big( - \tfrac{T}{8|\mathcal C|H_2 R} \Big)
\]
Combining both:
\[
3 |\mathcal C|R\exp\!\Big( - \tfrac{T}{\min(2|\mathcal P|H_1,8|\mathcal C|H_2) R} \Big)
\]
which gives the theorem.
\end{proof}

\setcounter{proposition}{1}
\begin{proposition}[Inference-Agnostic Optimality]
 The Inference-Agnostic prompt-optimization policy remains optimal under linear transformation of $R_{x}^{\textsc{IA}}(c,a)$, that is, $\sigma R_{x}^{\textsc{IA}}(c,a), \sigma > 0$ and an optimal policy can be recovered trivially from $Q$-function under affine transformation:
\[
    Q^{AF}(c,a)
      :=\mathbb{E}_{x\sim\mathcal{X}}\bigl[\sigma R_{x}^{\textsc{IA}}(c,a)+\mu\bigr] = \sigma Q^{\textsc{IA}}(c,a)+\mu .
      \quad
\]
\end{proposition}
\begin{proof}
By linearity of expectation,
\[
\begin{aligned}
Q'(c,a)
&= \mathbb{E}_{x}\!\left[\sigma R_{x}^{\textsc{IA}}(c,a)+\mu\right] \\
&= \sigma\,\mathbb{E}_{x}\!\left[R_{x}^{\textsc{IA}}(c,a)\right] + \mu \\
&= \sigma Q^{\textsc{IA}}(c,a) + \mu .
\end{aligned}
\]
Since $\sigma>0$, for any two arms $a,b$ we have
\[
\begin{aligned}
Q'(c,a) \ge Q'(c,b)
&\iff \sigma Q^{\textsc{IA}}(c,a)+\mu
      \ge \sigma Q^{\textsc{IA}}(c,b)+\mu \\
&\iff Q^{\textsc{IA}}(c,a) \ge Q^{\textsc{IA}}(c,b).
\end{aligned}
\]
Therefore the ordering of arms is preserved and the $\arg\max$ set is identical.
\end{proof}

\section*{Appendix B}

\paragraph{Synthetic-Bernoulli Environment.}
We consider a setting with $|\mathcal{P}| = 32$ prompts, each evaluated over a hidden mixture of query difficulty tiers—$\{\text{easy}, \text{medium}, \text{hard}\}$—spanning $|\mathcal{X}| = 520$ queries, with proportions $6:4:3$. For each prompt $p$ and query $x$, the single-shot success probability is denoted $q_p(x) \in [0,1]$.

A pull of $N \leq N_{\max}$ for prompt $p$ on query $x$ generates i.i.d.\ Bernoulli outcomes $\{z_i\}_{i=1}^N$ where $\Pr(z_i = 1) = q_p(x)$, and each completion incurs a per-completion cost $k_p$. The result is an array $\bigl[\;z_i,\;k_p\;\bigr]_{i=1}^N$.

Majority Voting (\textsc{MV}) sets $M = 1$ if $\sum_i z_i > N/2$, $M = 0$ if $\sum_i z_i < N/2$, and assigns $M = 0.5$ (by fair coin) in the case of a tie ($N$ even, $\sum_i z_i = N/2$).

The cost-adjusted utility for context $c \in \{\text{low}, \text{mid}, \text{high}\}$ is computed as
\[
u_c = w_1 M - w_2(c) \sum_{i=1}^N k_p,
\]
where $w_1 = 1$ and $w_2(c) \in \{0,\ 0.2,\ 1.0\}$ depending on the cost tier.

To instantiate the environment, we generate two prompt archetypes: \emph{deceiving prompts}, which achieve high average accuracy but exhibit low $q_p(x)$ on hard queries, and \emph{all-rounders}, which maintain moderate accuracy more uniformly across tiers. Per-prompt costs $k_p$ are sampled from a normal distribution with mean $0.02$ and variance $0.005$.

\paragraph{Synthetic-Categorical Environment.}
We model $|\mathcal{P}| = 32$ prompts, each paired with $|\mathcal{X}| = 512$ queries and $K = 2$ positive objectives. For every $(p, x)$, there are $M$ categorical outcomes, each represented by a vector $o_j \in \mathbb{R}^K$. A pull of $N \leq N_{\max}$ (with $N_{\max}=32$) for prompt $p$ on query $x$ generates $N$ i.i.d.\ outcome vectors, resulting in rows $[\;o_{i,1},\, o_{i,2},\, k_p\;]$, where $k_p$ denotes the per-completion cost for prompt $p$.

Given a context $c$ with weights $w = (w_1, w_2, w_{\text{cost}})$, where $w_1 + w_2 = 1$ and $w_{\text{cost}} \ge 0$, the Best-of-$N$ (\textsc{BoN}) utility is defined as
\[
u_c
= \max_{1 \leq i \leq N} \left( w_1 o_{i,1} + w_2 o_{i,2} \right)
- w_{\text{cost}}\, N\, k_p .
\]

To construct the environment, outcome vectors are sampled from $\{-4, \ldots, 4\}^2$. We instantiate two prompt archetypes: \emph{HMLV} (high mean, low variance; excels at $N{=}1$) and \emph{LMHV} (lower mean, high variance; benefits from larger $N$), each specializing in one objective. For each $(p, x)$, we add small per-query noise to the categorical outcome probabilities, introduce a mild train-to-test shift by perturbing these probabilities, sample per-prompt costs $k_p \in [0.02,\, 0.1]$, and draw context weights from a grid satisfying $w_1 + w_2 = 1$ with $w_{\text{cost}} \in \{0.1,\ 0.5,\ 1.0\}$.

\paragraph{MATH Environment.}
We select $316$ integer-answer problems from the MATH dataset%
\footnote{\url{https://huggingface.co/datasets/HuggingFaceH4/MATH-500}}.
A set of $25$ prompt templates is authored using \emph{ChatGPT-o3}. For each $(\text{prompt},\, \text{problem})$ pair, we sample $128$ responses from \texttt{Llama-3.3-70B-Instruct} at temperature $T = 0.7$, parsing each completion to its final integer answer.

The dataset is then processed as follows:
\begin{enumerate}
    \item For each problem, retain the global top-4 answers and group all other answers into a single \textsc{other} bucket (five categories in total).
    \item Compute per-prompt costs as the normalized average token length of its responses.
\end{enumerate}

This yields a categorical environment (analogous to the Synthetic-Categorical setting) with $|\mathcal{P}| = 25$, $N_{\max} = 32$, a uniform context prior $c \in \{\text{low},\,\text{mid},\,\text{high}\}$, and cost coefficients $\{0,\,0.2,\,1.0\}$. Utility is evaluated via \textsc{MV}.

\paragraph{CommonsenseQA Environment.}
We randomly sample $1{,}500$ multiple-choice questions from the CommonsenseQA corpus%
\footnote{\url{https://huggingface.co/datasets/tau/commonsense_qa}},
and author $48$ prompt templates using \emph{ChatGPT-o3}.
For each $(\text{prompt},\,\text{question})$ pair, we query \texttt{Llama-3.3-70B-Instruct} at temperature $T = 1.1$, collecting $128$ JSON-constrained answers (one of ``Option A''–``Option E''). Each prompt is assigned a constant cost $k_p = 0.01$ as we use JSON mode, and all completion has equal length.

The resulting data is used to construct a categorical environment (in analogy to the Synthetic-Categorical setting) with $|\mathcal{P}| = 48$, $N_{\max} = 32$, a uniform context prior, and cost coefficients $\{0,\,0.2,\,1.0\}$.

\paragraph{Helpful--Harmless Environment.}
We filter the \textsc{HH-RLHF} conversations%
\footnote{\url{https://huggingface.co/datasets/Anthropic/hh-rlhf}}
to the $1{,}355$ examples containing a single user query and a single assistant response.
Using \emph{ChatGPT-o3}, we craft $20$ prompt templates. For each $(\text{prompt},\,\text{query})$ pair, we sample $128$ continuations from \texttt{Llama-3.3-70B-Instruct} at temperature $T = 0.7$.
Each continuation is scored by separate public reward models~\cite{yang2024rewards} for \emph{helpfulness}%
\footnote{\url{Ray2333/gpt2-large-helpful-reward_model}}
and \emph{harmlessness}%
\footnote{\url{Ray2333/gpt2-large-harmless-reward_model}},
with scores normalized to $[-1,1]$.

The two reward scores are then binned on a $0.5$-spaced grid, producing a categorical distribution per $(\text{prompt},\,\text{query})$; per-prompt costs are computed as the average token length.
This data defines a categorical environment with $|\mathcal{P}|=20$, $N_{\max}=32$, and a uniform context prior over weight triples $(w_{\text{help}},\, w_{\text{harm}},\, w_{\text{cost}})$ with $w_{\text{help}} + w_{\text{harm}} = 1$ and $w_{\text{cost}} \in \{0.1,\, 0.5,\, 1.0\}$.

\paragraph{Summarization Environment.}
We randomly sample $1{,}201$ Reddit posts from the Summarize-from-Feedback corpus%
\footnote{\url{https://huggingface.co/datasets/openai/summarize_from_feedback}}
and design $20$ summarization prompt templates using \emph{ChatGPT-o3}.
For each $(\text{prompt},\,\text{post})$ pair, we query \texttt{Llama-3.3-70B-Instruct} at temperature $T = 0.7$ and collect $128$ candidate summaries.

Each summary is scored by two publicly available reward models: \emph{Preference}%
\footnote{\url{OpenAssistant/reward-model-deberta-v3-large-v2}}
and \emph{Faithful}%
\footnote{\url{CogComp/bart-faithful-summary-detector}},
with raw scores normalized to $[-1,1]$. We then bin each dimension in steps of $0.5$, producing a categorical distribution over the two reward dimensions, and compute per-prompt costs from average token length.

This data defines a categorical environment with $|\mathcal{P}|=20$, $N_{\max}=32$, and a uniform context prior over weight triples $(w_{\text{pref}},\, w_{\text{faith}},\, w_{\text{cost}})$ where $w_{\text{pref}} + w_{\text{faith}} = 1$ and $w_{\text{cost}} \in \{0.1,\, 0.5,\, 1.0\}$.

\textbf{Note:} All prompts are available under the prompts folder of the code base.

\section*{Appendix C}
\paragraph{Top-\(K\) screening.}
For the screening variant, we fixed \(K=4\) candidates after screening and swept the burn-in fraction \(\rho\in\{0.05,0.10,0.20,0.30,0.40\}\), which allocates a \(\rho\)-fraction of the budget to obtain initial estimates before trimming. The parameter-sweep protocol matched the baselines. We selected \(\rho=0.20\) for reporting, as it achieved the best overall performance while remaining robust across datasets and inference regimes (Table~\ref{tab:psst+k4_hp_paramT}).

\paragraph{UCB.}
We tuned the exploration constant over \(c\in\{0.1,0.5,1.0,2.0,4.0,8.0\}\) under the same budgets, using 20\% of the data per environment with identical seeds across settings, and \(10{,}000\) test contexts. The agent ranks arms by the standard UCB index
\[
\mathrm{UCB}_i(t)\;=\;\hat{\mu}_i(t)\;+\;c\,\sqrt{\frac{\ln t}{n_i(t)}},
\]
where \(\hat{\mu}_i(t)\) is the empirical mean utility of arm \(i\), \(n_i(t)\) its pull count, and \(t\) the total pulls. We selected \(c=0.1\) for reporting, as it achieved the best overall performance while remaining robust across datasets and inference regimes (Table~\ref{tab:ucb_hp_paramT}).

\paragraph{$\epsilon$-greedy.}
We swept the greedy probability \(1-\epsilon\in\{0.50,0.75,0.80,0.85,0.90,0.95\}\) separately for each dataset and inference regime (\textsc{MV}, \textsc{BoN}). For every \(\epsilon\), agents were trained under budgets \(T\in\{3\text{K},5\text{K},10\text{K},20\text{K},30\text{K},40\text{K}\}\), using 20\% of the data per environment with deterministic reseeding; evaluation used \(10{,}000\) test contexts per environment. We selected \(\epsilon=0.15\) for reporting, as it achieved the best overall performance while remaining robust across datasets and inference regimes (Table~\ref{tab:greedy_hp_paramT}).

\begin{table*}[t]
\centering
\small
\begin{tabular}{lcccccc}
\toprule
Param $\times$ T & HH & Summarization & SC & SB & MATH & CQA \\
\midrule
$\rho$=0.05, $T=3000$ & 0.40 $\pm$ 0.00 & 0.20 $\pm$ 0.00 & 2.77 $\pm$ 0.02 & 0.83 $\pm$ 0.00 & 0.79 $\pm$ 0.01 & 0.75 $\pm$ 0.00 \\
$\rho$=0.05, $T=5000$ & 0.40 $\pm$ 0.00 & 0.22 $\pm$ 0.00 & 2.83 $\pm$ 0.02 & 0.83 $\pm$ 0.00 & 0.81 $\pm$ 0.01 & 0.76 $\pm$ 0.00 \\
$\rho$=0.05, $T=10000$ & 0.42 $\pm$ 0.00 & 0.21 $\pm$ 0.00 & 2.83 $\pm$ 0.02 & 0.85 $\pm$ 0.01 & 0.81 $\pm$ 0.01 & 0.76 $\pm$ 0.00 \\
$\rho$=0.05, $T=20000$ & 0.43 $\pm$ 0.00 & 0.22 $\pm$ 0.00 & 2.87 $\pm$ 0.02 & 0.84 $\pm$ 0.00 & 0.80 $\pm$ 0.01 & 0.76 $\pm$ 0.00 \\
$\rho$=0.05, $T=30000$ & 0.44 $\pm$ 0.00 & 0.23 $\pm$ 0.00 & 2.88 $\pm$ 0.01 & 0.84 $\pm$ 0.00 & 0.82 $\pm$ 0.01 & 0.77 $\pm$ 0.00 \\
$\rho$=0.05, $T=40000$ & 0.43 $\pm$ 0.00 & 0.23 $\pm$ 0.00 & 2.87 $\pm$ 0.02 & 0.84 $\pm$ 0.00 & 0.81 $\pm$ 0.00 & 0.77 $\pm$ 0.00 \\
$\rho$=0.10, $T=3000$ & 0.41 $\pm$ 0.00 & 0.21 $\pm$ 0.00 & 2.79 $\pm$ 0.02 & 0.83 $\pm$ 0.00 & 0.80 $\pm$ 0.01 & 0.76 $\pm$ 0.00 \\
$\rho$=0.10, $T=5000$ & 0.41 $\pm$ 0.00 & 0.22 $\pm$ 0.00 & 2.84 $\pm$ 0.02 & 0.84 $\pm$ 0.00 & 0.81 $\pm$ 0.01 & 0.76 $\pm$ 0.00 \\
$\rho$=0.10, $T=10000$ & 0.42 $\pm$ 0.00 & 0.21 $\pm$ 0.00 & 2.86 $\pm$ 0.02 & 0.84 $\pm$ 0.00 & 0.81 $\pm$ 0.01 & 0.77 $\pm$ 0.00 \\
$\rho$=0.10, $T=20000$ & 0.43 $\pm$ 0.00 & 0.23 $\pm$ 0.00 & 2.88 $\pm$ 0.02 & 0.84 $\pm$ 0.00 & 0.81 $\pm$ 0.01 & 0.77 $\pm$ 0.00 \\
$\rho$=0.10, $T=30000$ & 0.44 $\pm$ 0.00 & 0.23 $\pm$ 0.00 & 2.89 $\pm$ 0.02 & 0.84 $\pm$ 0.00 & 0.81 $\pm$ 0.01 & 0.77 $\pm$ 0.01 \\
$\rho$=0.10, $T=40000$ & 0.44 $\pm$ 0.00 & 0.23 $\pm$ 0.00 & 2.88 $\pm$ 0.02 & 0.84 $\pm$ 0.00 & 0.82 $\pm$ 0.00 & 0.77 $\pm$ 0.00 \\
$\rho$=0.20, $T=3000$ & 0.41 $\pm$ 0.00 & 0.20 $\pm$ 0.00 & 2.77 $\pm$ 0.02 & 0.83 $\pm$ 0.00 & 0.80 $\pm$ 0.01 & 0.76 $\pm$ 0.00 \\
$\rho$=0.20, $T=5000$ & 0.41 $\pm$ 0.00 & 0.22 $\pm$ 0.00 & 2.84 $\pm$ 0.02 & 0.83 $\pm$ 0.00 & 0.80 $\pm$ 0.01 & 0.76 $\pm$ 0.00 \\
$\rho$=0.20, $T=10000$ & 0.43 $\pm$ 0.00 & 0.22 $\pm$ 0.00 & 2.85 $\pm$ 0.02 & 0.83 $\pm$ 0.00 & 0.81 $\pm$ 0.01 & 0.76 $\pm$ 0.00 \\
$\rho$=0.20, $T=20000$ & 0.43 $\pm$ 0.00 & 0.23 $\pm$ 0.00 & 2.87 $\pm$ 0.01 & 0.84 $\pm$ 0.00 & 0.81 $\pm$ 0.01 & 0.77 $\pm$ 0.00 \\
$\rho$=0.20, $T=30000$ & 0.44 $\pm$ 0.00 & 0.23 $\pm$ 0.00 & 2.89 $\pm$ 0.02 & 0.84 $\pm$ 0.00 & 0.82 $\pm$ 0.01 & 0.76 $\pm$ 0.00 \\
$\rho$=0.20, $T=40000$ & 0.44 $\pm$ 0.00 & 0.22 $\pm$ 0.00 & 2.88 $\pm$ 0.02 & 0.84 $\pm$ 0.00 & 0.82 $\pm$ 0.01 & 0.77 $\pm$ 0.00 \\
$\rho$=0.30, $T=3000$ & 0.41 $\pm$ 0.00 & 0.21 $\pm$ 0.00 & 2.81 $\pm$ 0.02 & 0.83 $\pm$ 0.00 & 0.80 $\pm$ 0.01 & 0.76 $\pm$ 0.00 \\
$\rho$=0.30, $T=5000$ & 0.41 $\pm$ 0.00 & 0.22 $\pm$ 0.00 & 2.85 $\pm$ 0.01 & 0.83 $\pm$ 0.00 & 0.81 $\pm$ 0.01 & 0.76 $\pm$ 0.00 \\
$\rho$=0.30, $T=10000$ & 0.42 $\pm$ 0.00 & 0.22 $\pm$ 0.00 & 2.85 $\pm$ 0.02 & 0.84 $\pm$ 0.00 & 0.81 $\pm$ 0.01 & 0.76 $\pm$ 0.00 \\
$\rho$=0.30, $T=20000$ & 0.43 $\pm$ 0.00 & 0.22 $\pm$ 0.00 & 2.88 $\pm$ 0.01 & 0.83 $\pm$ 0.00 & 0.81 $\pm$ 0.01 & 0.76 $\pm$ 0.00 \\
$\rho$=0.30, $T=30000$ & 0.44 $\pm$ 0.00 & 0.22 $\pm$ 0.00 & 2.88 $\pm$ 0.01 & 0.84 $\pm$ 0.00 & 0.82 $\pm$ 0.01 & 0.76 $\pm$ 0.00 \\
$\rho$=0.30, $T=40000$ & 0.44 $\pm$ 0.00 & 0.23 $\pm$ 0.00 & 2.89 $\pm$ 0.01 & 0.84 $\pm$ 0.00 & 0.82 $\pm$ 0.01 & 0.77 $\pm$ 0.00 \\
$\rho$=0.40, $T=3000$ & 0.41 $\pm$ 0.00 & 0.20 $\pm$ 0.00 & 2.80 $\pm$ 0.01 & 0.83 $\pm$ 0.00 & 0.80 $\pm$ 0.01 & 0.76 $\pm$ 0.00 \\
$\rho$=0.40, $T=5000$ & 0.42 $\pm$ 0.00 & 0.20 $\pm$ 0.00 & 2.80 $\pm$ 0.02 & 0.83 $\pm$ 0.00 & 0.80 $\pm$ 0.01 & 0.76 $\pm$ 0.00 \\
$\rho$=0.40, $T=10000$ & 0.43 $\pm$ 0.00 & 0.22 $\pm$ 0.00 & 2.85 $\pm$ 0.01 & 0.83 $\pm$ 0.00 & 0.81 $\pm$ 0.01 & 0.77 $\pm$ 0.00 \\
$\rho$=0.40, $T=20000$ & 0.43 $\pm$ 0.00 & 0.22 $\pm$ 0.00 & 2.87 $\pm$ 0.02 & 0.83 $\pm$ 0.00 & 0.81 $\pm$ 0.01 & 0.77 $\pm$ 0.00 \\
$\rho$=0.40, $T=30000$ & 0.44 $\pm$ 0.00 & 0.22 $\pm$ 0.00 & 2.88 $\pm$ 0.01 & 0.83 $\pm$ 0.00 & 0.81 $\pm$ 0.01 & 0.77 $\pm$ 0.00 \\
$\rho$=0.40, $T=40000$ & 0.44 $\pm$ 0.00 & 0.23 $\pm$ 0.00 & 2.89 $\pm$ 0.01 & 0.84 $\pm$ 0.00 & 0.82 $\pm$ 0.01 & 0.77 $\pm$ 0.00 \\
\bottomrule
\end{tabular}
\caption{PSST+K4: mean $\pm$ SEM across datasets (rows are param, $\rho$ and $T$).}
\label{tab:psst+k4_hp_paramT}
\end{table*}

\begin{table*}[t]
\centering
\small
\begin{tabular}{lcccccc}
\toprule
Param $\times$ T & HH & Summarization & SC & SB & MATH & CQA \\
\midrule
c=0.1, $T=3000$ & 0.38 $\pm$ 0.00 & 0.19 $\pm$ 0.01 & 2.83 $\pm$ 0.02 & 0.82 $\pm$ 0.00 & 0.78 $\pm$ 0.01 & 0.75 $\pm$ 0.00 \\
c=0.1, $T=5000$ & 0.39 $\pm$ 0.00 & 0.21 $\pm$ 0.00 & 2.88 $\pm$ 0.01 & 0.83 $\pm$ 0.00 & 0.80 $\pm$ 0.01 & 0.75 $\pm$ 0.00 \\
c=0.1, $T=10000$ & 0.41 $\pm$ 0.00 & 0.23 $\pm$ 0.00 & 2.94 $\pm$ 0.01 & 0.83 $\pm$ 0.00 & 0.80 $\pm$ 0.01 & 0.76 $\pm$ 0.00 \\
c=0.1, $T=20000$ & 0.43 $\pm$ 0.00 & 0.25 $\pm$ 0.00 & 2.98 $\pm$ 0.01 & 0.86 $\pm$ 0.00 & 0.80 $\pm$ 0.01 & 0.76 $\pm$ 0.00 \\
c=0.1, $T=30000$ & 0.43 $\pm$ 0.00 & 0.25 $\pm$ 0.00 & 2.99 $\pm$ 0.01 & 0.88 $\pm$ 0.00 & 0.81 $\pm$ 0.01 & 0.76 $\pm$ 0.01 \\
c=0.1, $T=40000$ & 0.44 $\pm$ 0.00 & 0.25 $\pm$ 0.00 & 3.00 $\pm$ 0.01 & 0.89 $\pm$ 0.00 & 0.81 $\pm$ 0.01 & 0.76 $\pm$ 0.00 \\
c=0.5, $T=3000$ & 0.37 $\pm$ 0.00 & 0.19 $\pm$ 0.01 & 2.85 $\pm$ 0.02 & 0.82 $\pm$ 0.00 & 0.78 $\pm$ 0.01 & 0.75 $\pm$ 0.00 \\
c=0.5, $T=5000$ & 0.38 $\pm$ 0.00 & 0.20 $\pm$ 0.00 & 2.90 $\pm$ 0.01 & 0.82 $\pm$ 0.00 & 0.79 $\pm$ 0.01 & 0.75 $\pm$ 0.00 \\
c=0.5, $T=10000$ & 0.41 $\pm$ 0.00 & 0.23 $\pm$ 0.00 & 2.94 $\pm$ 0.01 & 0.83 $\pm$ 0.00 & 0.80 $\pm$ 0.01 & 0.76 $\pm$ 0.00 \\
c=0.5, $T=20000$ & 0.43 $\pm$ 0.00 & 0.24 $\pm$ 0.00 & 2.98 $\pm$ 0.01 & 0.84 $\pm$ 0.00 & 0.80 $\pm$ 0.01 & 0.75 $\pm$ 0.00 \\
c=0.5, $T=30000$ & 0.43 $\pm$ 0.00 & 0.24 $\pm$ 0.00 & 3.00 $\pm$ 0.01 & 0.86 $\pm$ 0.00 & 0.80 $\pm$ 0.01 & 0.75 $\pm$ 0.00 \\
c=0.5, $T=40000$ & 0.44 $\pm$ 0.00 & 0.25 $\pm$ 0.00 & 3.00 $\pm$ 0.01 & 0.88 $\pm$ 0.00 & 0.81 $\pm$ 0.01 & 0.75 $\pm$ 0.00 \\
c=1.0, $T=3000$ & 0.37 $\pm$ 0.00 & 0.19 $\pm$ 0.01 & 2.88 $\pm$ 0.02 & 0.82 $\pm$ 0.00 & 0.78 $\pm$ 0.01 & 0.75 $\pm$ 0.00 \\
c=1.0, $T=5000$ & 0.37 $\pm$ 0.00 & 0.19 $\pm$ 0.01 & 2.91 $\pm$ 0.02 & 0.83 $\pm$ 0.00 & 0.79 $\pm$ 0.01 & 0.75 $\pm$ 0.00 \\
c=1.0, $T=10000$ & 0.41 $\pm$ 0.00 & 0.22 $\pm$ 0.00 & 2.94 $\pm$ 0.01 & 0.83 $\pm$ 0.00 & 0.81 $\pm$ 0.01 & 0.76 $\pm$ 0.00 \\
c=1.0, $T=20000$ & 0.42 $\pm$ 0.00 & 0.24 $\pm$ 0.00 & 2.98 $\pm$ 0.01 & 0.84 $\pm$ 0.00 & 0.80 $\pm$ 0.01 & 0.75 $\pm$ 0.00 \\
c=1.0, $T=30000$ & 0.43 $\pm$ 0.00 & 0.24 $\pm$ 0.00 & 3.00 $\pm$ 0.01 & 0.86 $\pm$ 0.01 & 0.81 $\pm$ 0.01 & 0.76 $\pm$ 0.00 \\
c=1.0, $T=40000$ & 0.43 $\pm$ 0.00 & 0.25 $\pm$ 0.00 & 3.00 $\pm$ 0.01 & 0.87 $\pm$ 0.00 & 0.81 $\pm$ 0.01 & 0.76 $\pm$ 0.01 \\
c=2.0, $T=3000$ & 0.37 $\pm$ 0.00 & 0.18 $\pm$ 0.01 & 2.86 $\pm$ 0.02 & 0.82 $\pm$ 0.00 & 0.78 $\pm$ 0.01 & 0.75 $\pm$ 0.00 \\
c=2.0, $T=5000$ & 0.38 $\pm$ 0.00 & 0.19 $\pm$ 0.01 & 2.93 $\pm$ 0.01 & 0.83 $\pm$ 0.00 & 0.79 $\pm$ 0.01 & 0.76 $\pm$ 0.00 \\
c=2.0, $T=10000$ & 0.40 $\pm$ 0.00 & 0.23 $\pm$ 0.00 & 2.94 $\pm$ 0.01 & 0.84 $\pm$ 0.00 & 0.80 $\pm$ 0.01 & 0.76 $\pm$ 0.00 \\
c=2.0, $T=20000$ & 0.42 $\pm$ 0.00 & 0.24 $\pm$ 0.00 & 2.98 $\pm$ 0.01 & 0.85 $\pm$ 0.00 & 0.80 $\pm$ 0.01 & 0.75 $\pm$ 0.00 \\
c=2.0, $T=30000$ & 0.42 $\pm$ 0.00 & 0.24 $\pm$ 0.00 & 2.99 $\pm$ 0.01 & 0.86 $\pm$ 0.00 & 0.81 $\pm$ 0.01 & 0.76 $\pm$ 0.00 \\
c=2.0, $T=40000$ & 0.43 $\pm$ 0.00 & 0.25 $\pm$ 0.00 & 2.99 $\pm$ 0.01 & 0.88 $\pm$ 0.00 & 0.81 $\pm$ 0.01 & 0.75 $\pm$ 0.00 \\
c=4.0, $T=3000$ & 0.37 $\pm$ 0.00 & 0.18 $\pm$ 0.01 & 2.85 $\pm$ 0.02 & 0.82 $\pm$ 0.00 & 0.78 $\pm$ 0.01 & 0.75 $\pm$ 0.00 \\
c=4.0, $T=5000$ & 0.37 $\pm$ 0.00 & 0.18 $\pm$ 0.00 & 2.91 $\pm$ 0.01 & 0.83 $\pm$ 0.00 & 0.79 $\pm$ 0.01 & 0.76 $\pm$ 0.00 \\
c=4.0, $T=10000$ & 0.41 $\pm$ 0.00 & 0.22 $\pm$ 0.00 & 2.94 $\pm$ 0.01 & 0.84 $\pm$ 0.00 & 0.81 $\pm$ 0.01 & 0.76 $\pm$ 0.00 \\
c=4.0, $T=20000$ & 0.42 $\pm$ 0.00 & 0.24 $\pm$ 0.00 & 2.97 $\pm$ 0.01 & 0.84 $\pm$ 0.00 & 0.80 $\pm$ 0.01 & 0.75 $\pm$ 0.00 \\
c=4.0, $T=30000$ & 0.42 $\pm$ 0.00 & 0.24 $\pm$ 0.00 & 2.99 $\pm$ 0.01 & 0.86 $\pm$ 0.00 & 0.81 $\pm$ 0.01 & 0.75 $\pm$ 0.00 \\
c=4.0, $T=40000$ & 0.43 $\pm$ 0.00 & 0.25 $\pm$ 0.00 & 3.00 $\pm$ 0.01 & 0.87 $\pm$ 0.00 & 0.81 $\pm$ 0.01 & 0.76 $\pm$ 0.00 \\
c=8.0, $T=3000$ & 0.37 $\pm$ 0.00 & 0.18 $\pm$ 0.01 & 2.86 $\pm$ 0.02 & 0.82 $\pm$ 0.00 & 0.78 $\pm$ 0.01 & 0.75 $\pm$ 0.00 \\
c=8.0, $T=5000$ & 0.38 $\pm$ 0.00 & 0.19 $\pm$ 0.01 & 2.90 $\pm$ 0.02 & 0.83 $\pm$ 0.00 & 0.79 $\pm$ 0.01 & 0.76 $\pm$ 0.00 \\
c=8.0, $T=10000$ & 0.40 $\pm$ 0.00 & 0.22 $\pm$ 0.00 & 2.92 $\pm$ 0.01 & 0.84 $\pm$ 0.00 & 0.80 $\pm$ 0.01 & 0.76 $\pm$ 0.00 \\
c=8.0, $T=20000$ & 0.42 $\pm$ 0.00 & 0.24 $\pm$ 0.00 & 2.97 $\pm$ 0.01 & 0.84 $\pm$ 0.00 & 0.80 $\pm$ 0.01 & 0.76 $\pm$ 0.00 \\
c=8.0, $T=30000$ & 0.43 $\pm$ 0.00 & 0.24 $\pm$ 0.00 & 2.98 $\pm$ 0.01 & 0.86 $\pm$ 0.00 & 0.81 $\pm$ 0.01 & 0.75 $\pm$ 0.00 \\
c=8.0, $T=40000$ & 0.43 $\pm$ 0.00 & 0.25 $\pm$ 0.00 & 2.99 $\pm$ 0.01 & 0.87 $\pm$ 0.00 & 0.81 $\pm$ 0.01 & 0.76 $\pm$ 0.01 \\
\bottomrule
\end{tabular}
\caption{UCB: mean $\pm$ SEM across datasets (rows are param, $T$).}
\label{tab:ucb_hp_paramT}
\end{table*}

\begin{table*}[t]
\centering
\small
\begin{tabular}{lcccccc}
\toprule
Param $\times$ T & HH & Summarization & SC & SB & MATH & CQA \\
\midrule
e=0.50, $T=3000$ & 0.37 $\pm$ 0.00 & 0.17 $\pm$ 0.01 & 2.78 $\pm$ 0.02 & 0.83 $\pm$ 0.00 & 0.79 $\pm$ 0.01 & 0.76 $\pm$ 0.00 \\
e=0.50, $T=5000$ & 0.39 $\pm$ 0.00 & 0.20 $\pm$ 0.01 & 2.82 $\pm$ 0.01 & 0.83 $\pm$ 0.00 & 0.80 $\pm$ 0.01 & 0.75 $\pm$ 0.00 \\
e=0.50, $T=10000$ & 0.41 $\pm$ 0.00 & 0.21 $\pm$ 0.00 & 2.90 $\pm$ 0.01 & 0.84 $\pm$ 0.00 & 0.79 $\pm$ 0.01 & 0.75 $\pm$ 0.00 \\
e=0.50, $T=20000$ & 0.42 $\pm$ 0.00 & 0.23 $\pm$ 0.00 & 2.94 $\pm$ 0.01 & 0.84 $\pm$ 0.00 & 0.79 $\pm$ 0.01 & 0.74 $\pm$ 0.00 \\
e=0.50, $T=30000$ & 0.43 $\pm$ 0.00 & 0.24 $\pm$ 0.00 & 2.97 $\pm$ 0.01 & 0.85 $\pm$ 0.00 & 0.80 $\pm$ 0.01 & 0.75 $\pm$ 0.00 \\
e=0.50, $T=40000$ & 0.43 $\pm$ 0.00 & 0.25 $\pm$ 0.00 & 2.98 $\pm$ 0.01 & 0.85 $\pm$ 0.00 & 0.80 $\pm$ 0.01 & 0.75 $\pm$ 0.01 \\
e=0.75, $T=3000$ & 0.38 $\pm$ 0.00 & 0.16 $\pm$ 0.01 & 2.75 $\pm$ 0.03 & 0.83 $\pm$ 0.00 & 0.79 $\pm$ 0.01 & 0.75 $\pm$ 0.00 \\
e=0.75, $T=5000$ & 0.39 $\pm$ 0.00 & 0.17 $\pm$ 0.01 & 2.86 $\pm$ 0.01 & 0.84 $\pm$ 0.00 & 0.79 $\pm$ 0.01 & 0.76 $\pm$ 0.00 \\
e=0.75, $T=10000$ & 0.40 $\pm$ 0.00 & 0.20 $\pm$ 0.01 & 2.91 $\pm$ 0.01 & 0.84 $\pm$ 0.00 & 0.79 $\pm$ 0.01 & 0.76 $\pm$ 0.00 \\
e=0.75, $T=20000$ & 0.42 $\pm$ 0.00 & 0.23 $\pm$ 0.00 & 2.95 $\pm$ 0.01 & 0.84 $\pm$ 0.00 & 0.79 $\pm$ 0.00 & 0.74 $\pm$ 0.00 \\
e=0.75, $T=30000$ & 0.43 $\pm$ 0.00 & 0.23 $\pm$ 0.00 & 2.97 $\pm$ 0.01 & 0.85 $\pm$ 0.00 & 0.81 $\pm$ 0.01 & 0.75 $\pm$ 0.00 \\
e=0.75, $T=40000$ & 0.43 $\pm$ 0.00 & 0.24 $\pm$ 0.00 & 2.96 $\pm$ 0.01 & 0.85 $\pm$ 0.00 & 0.80 $\pm$ 0.01 & 0.74 $\pm$ 0.00 \\
e=0.80, $T=3000$ & 0.38 $\pm$ 0.00 & 0.18 $\pm$ 0.01 & 2.83 $\pm$ 0.02 & 0.83 $\pm$ 0.00 & 0.79 $\pm$ 0.01 & 0.75 $\pm$ 0.00 \\
e=0.80, $T=5000$ & 0.39 $\pm$ 0.00 & 0.19 $\pm$ 0.00 & 2.86 $\pm$ 0.02 & 0.83 $\pm$ 0.00 & 0.80 $\pm$ 0.01 & 0.76 $\pm$ 0.00 \\
e=0.80, $T=10000$ & 0.40 $\pm$ 0.00 & 0.19 $\pm$ 0.01 & 2.91 $\pm$ 0.01 & 0.84 $\pm$ 0.00 & 0.80 $\pm$ 0.01 & 0.76 $\pm$ 0.00 \\
e=0.80, $T=20000$ & 0.42 $\pm$ 0.00 & 0.23 $\pm$ 0.00 & 2.94 $\pm$ 0.01 & 0.84 $\pm$ 0.00 & 0.79 $\pm$ 0.01 & 0.75 $\pm$ 0.00 \\
e=0.80, $T=30000$ & 0.41 $\pm$ 0.00 & 0.23 $\pm$ 0.00 & 2.96 $\pm$ 0.01 & 0.84 $\pm$ 0.00 & 0.79 $\pm$ 0.01 & 0.75 $\pm$ 0.00 \\
e=0.80, $T=40000$ & 0.43 $\pm$ 0.00 & 0.24 $\pm$ 0.00 & 2.98 $\pm$ 0.01 & 0.85 $\pm$ 0.00 & 0.80 $\pm$ 0.01 & 0.74 $\pm$ 0.00 \\
e=0.85, $T=3000$ & 0.38 $\pm$ 0.00 & 0.16 $\pm$ 0.01 & 2.72 $\pm$ 0.04 & 0.83 $\pm$ 0.00 & 0.78 $\pm$ 0.01 & 0.75 $\pm$ 0.00 \\
e=0.85, $T=5000$ & 0.38 $\pm$ 0.00 & 0.17 $\pm$ 0.01 & 2.87 $\pm$ 0.01 & 0.83 $\pm$ 0.00 & 0.80 $\pm$ 0.01 & 0.76 $\pm$ 0.00 \\
e=0.85, $T=10000$ & 0.40 $\pm$ 0.00 & 0.20 $\pm$ 0.00 & 2.90 $\pm$ 0.01 & 0.84 $\pm$ 0.00 & 0.80 $\pm$ 0.01 & 0.76 $\pm$ 0.00 \\
e=0.85, $T=20000$ & 0.41 $\pm$ 0.00 & 0.22 $\pm$ 0.00 & 2.95 $\pm$ 0.01 & 0.84 $\pm$ 0.00 & 0.79 $\pm$ 0.01 & 0.75 $\pm$ 0.00 \\
e=0.85, $T=30000$ & 0.42 $\pm$ 0.00 & 0.23 $\pm$ 0.01 & 2.95 $\pm$ 0.01 & 0.85 $\pm$ 0.00 & 0.79 $\pm$ 0.01 & 0.75 $\pm$ 0.00 \\
e=0.85, $T=40000$ & 0.42 $\pm$ 0.00 & 0.24 $\pm$ 0.00 & 2.97 $\pm$ 0.01 & 0.85 $\pm$ 0.00 & 0.80 $\pm$ 0.01 & 0.75 $\pm$ 0.01 \\
e=0.90, $T=3000$ & 0.37 $\pm$ 0.00 & 0.17 $\pm$ 0.01 & 2.81 $\pm$ 0.03 & 0.82 $\pm$ 0.00 & 0.78 $\pm$ 0.01 & 0.75 $\pm$ 0.00 \\
e=0.90, $T=5000$ & 0.38 $\pm$ 0.00 & 0.17 $\pm$ 0.01 & 2.87 $\pm$ 0.02 & 0.83 $\pm$ 0.00 & 0.79 $\pm$ 0.01 & 0.76 $\pm$ 0.00 \\
e=0.90, $T=10000$ & 0.40 $\pm$ 0.00 & 0.19 $\pm$ 0.01 & 2.90 $\pm$ 0.01 & 0.84 $\pm$ 0.00 & 0.80 $\pm$ 0.01 & 0.76 $\pm$ 0.00 \\
e=0.90, $T=20000$ & 0.41 $\pm$ 0.00 & 0.22 $\pm$ 0.00 & 2.94 $\pm$ 0.01 & 0.84 $\pm$ 0.00 & 0.79 $\pm$ 0.01 & 0.75 $\pm$ 0.00 \\
e=0.90, $T=30000$ & 0.42 $\pm$ 0.00 & 0.23 $\pm$ 0.00 & 2.95 $\pm$ 0.01 & 0.85 $\pm$ 0.00 & 0.79 $\pm$ 0.01 & 0.75 $\pm$ 0.00 \\
e=0.90, $T=40000$ & 0.42 $\pm$ 0.00 & 0.24 $\pm$ 0.00 & 2.97 $\pm$ 0.01 & 0.85 $\pm$ 0.00 & 0.80 $\pm$ 0.01 & 0.75 $\pm$ 0.00 \\
e=0.95, $T=3000$ & 0.37 $\pm$ 0.00 & 0.17 $\pm$ 0.01 & 2.76 $\pm$ 0.03 & 0.82 $\pm$ 0.00 & 0.78 $\pm$ 0.01 & 0.75 $\pm$ 0.00 \\
e=0.95, $T=5000$ & 0.38 $\pm$ 0.00 & 0.17 $\pm$ 0.01 & 2.86 $\pm$ 0.01 & 0.83 $\pm$ 0.00 & 0.79 $\pm$ 0.01 & 0.76 $\pm$ 0.00 \\
e=0.95, $T=10000$ & 0.39 $\pm$ 0.00 & 0.19 $\pm$ 0.01 & 2.92 $\pm$ 0.01 & 0.84 $\pm$ 0.00 & 0.80 $\pm$ 0.01 & 0.76 $\pm$ 0.00 \\
e=0.95, $T=20000$ & 0.41 $\pm$ 0.00 & 0.21 $\pm$ 0.00 & 2.95 $\pm$ 0.01 & 0.84 $\pm$ 0.00 & 0.79 $\pm$ 0.01 & 0.75 $\pm$ 0.00 \\
e=0.95, $T=30000$ & 0.42 $\pm$ 0.00 & 0.22 $\pm$ 0.00 & 2.95 $\pm$ 0.01 & 0.85 $\pm$ 0.00 & 0.80 $\pm$ 0.01 & 0.75 $\pm$ 0.00 \\
e=0.95, $T=40000$ & 0.41 $\pm$ 0.00 & 0.23 $\pm$ 0.00 & 2.95 $\pm$ 0.01 & 0.84 $\pm$ 0.00 & 0.79 $\pm$ 0.00 & 0.75 $\pm$ 0.00 \\
\bottomrule
\end{tabular}
\caption{$\epsilon$-greedy: mean $\pm$ SEM across datasets (rows are param, $T$).}
\label{tab:greedy_hp_paramT}
\end{table*}

\section*{Appendix D}
\paragraph{Statistical testing.}
For each dataset and budget \(T\), we perform all pairwise algorithm comparisons using per-environment utilities as \emph{paired} samples (identical train/test splits via deterministic reseeding). Our default test is the two-sided Wilcoxon signed-rank test, which we apply to the aligned vectors after removing non-finite values and dropping exact ties (\texttt{zero\_method=wilcox}, \texttt{mode=auto}); pairs with fewer than two effective samples are skipped. When requested, we also report the paired sign test (binomial test on the sign of differences) after removing ties. To control multiplicity within each \((\text{dataset}, T)\) grid, we use Holm–Bonferroni adjustment by default (with options for Benjamini–Hochberg FDR or no correction). We declare a \emph{winner} if the adjusted \(p<\alpha=0.05\); the direction is determined by the sign of the median difference \(\operatorname{median}(x-y)\). In case of unequal environment counts across algorithms, samples are truncated to the minimum length to preserve pairing. Figures visualize the outcome matrix with entries in \(\{-1,0,+1\}\) indicating row-algorithm loss, non-significance, or win against the column algorithm, respectively. 

All the results are shown in Figs~\ref{fig:cqa_mv_pairwise_grid},~\ref{fig:hh_bon_pairwise_grid},~\ref{fig:math_mv_pairwise_grid},~\ref{fig:sb_mv_pairwise_grid},~\ref{fig:sc_bon_pairwise_grid},
~\ref{fig:summarization_bon_pairwise_grid}. Across all six datasets, we observe that PSST and the Top-K screening heuristic consistently outperform competing methods across most budget settings, with statistical significance.

\begin{figure*}[t]
\centering
\begin{tabular}{@{}ccc@{}}
\includegraphics[width=0.32\textwidth]{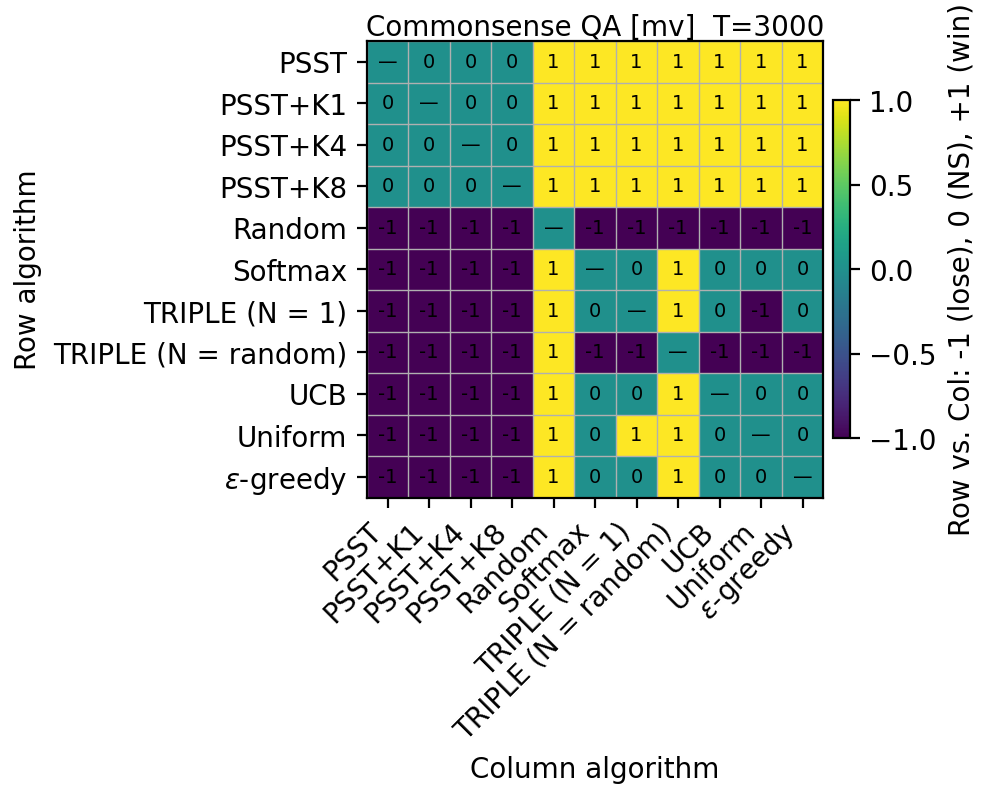} & \includegraphics[width=0.32\textwidth]{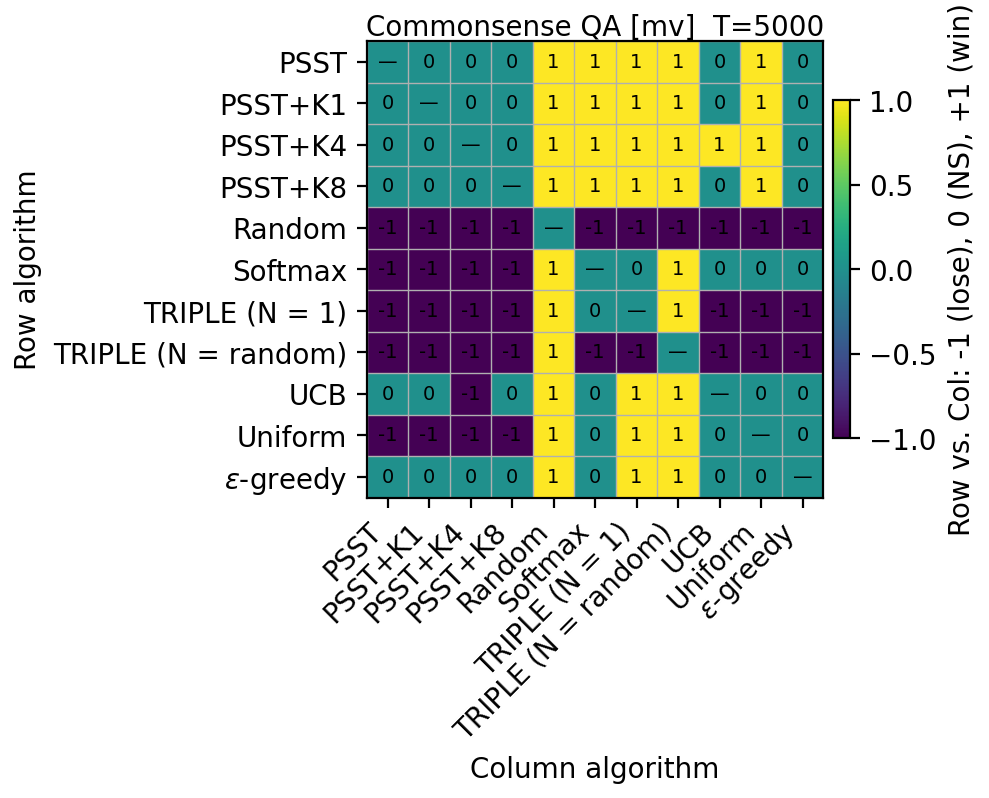} & \includegraphics[width=0.32\textwidth]{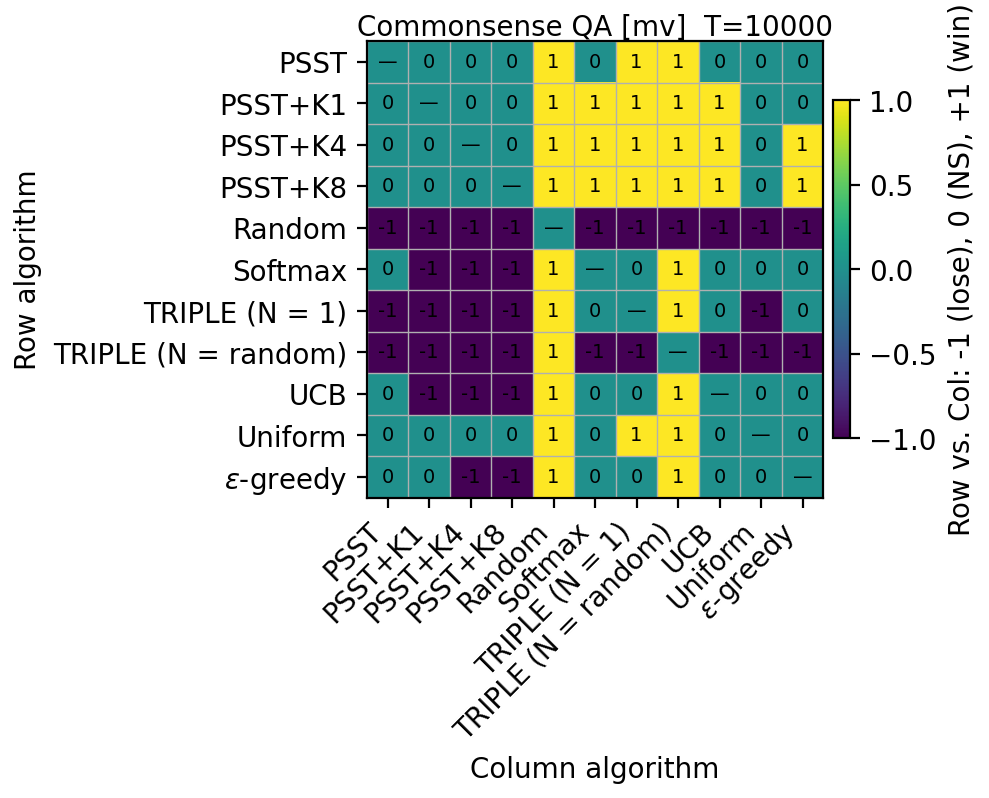} \\
[4pt]
\includegraphics[width=0.32\textwidth]{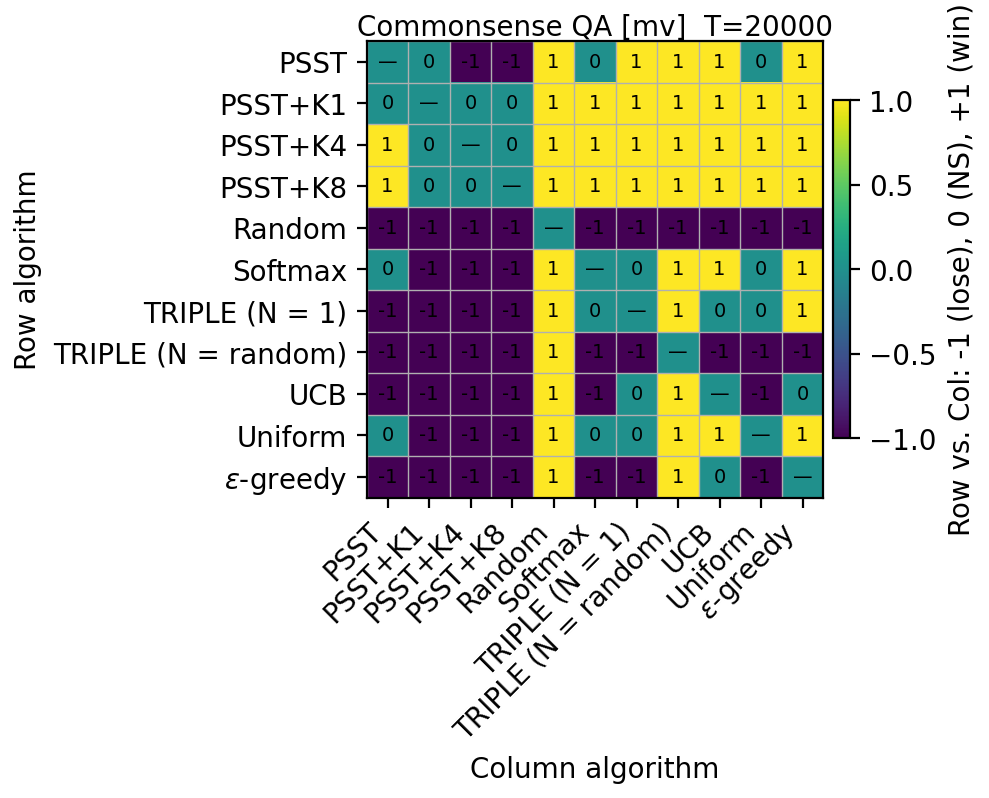} & \includegraphics[width=0.32\textwidth]{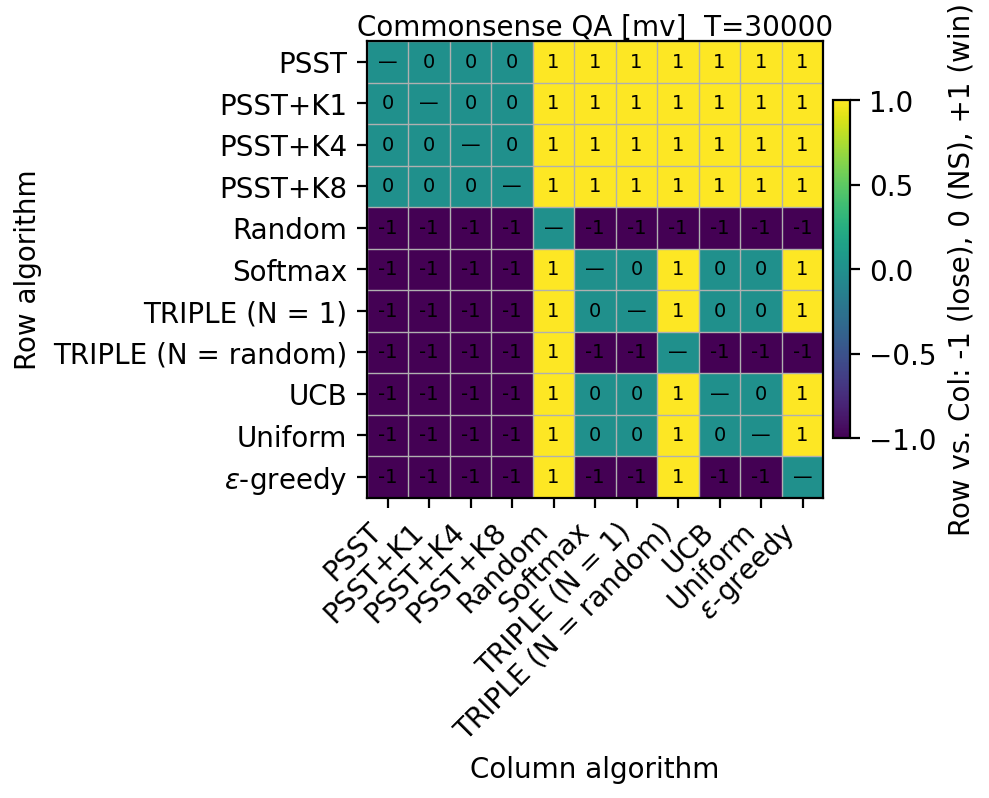} & \includegraphics[width=0.32\textwidth]{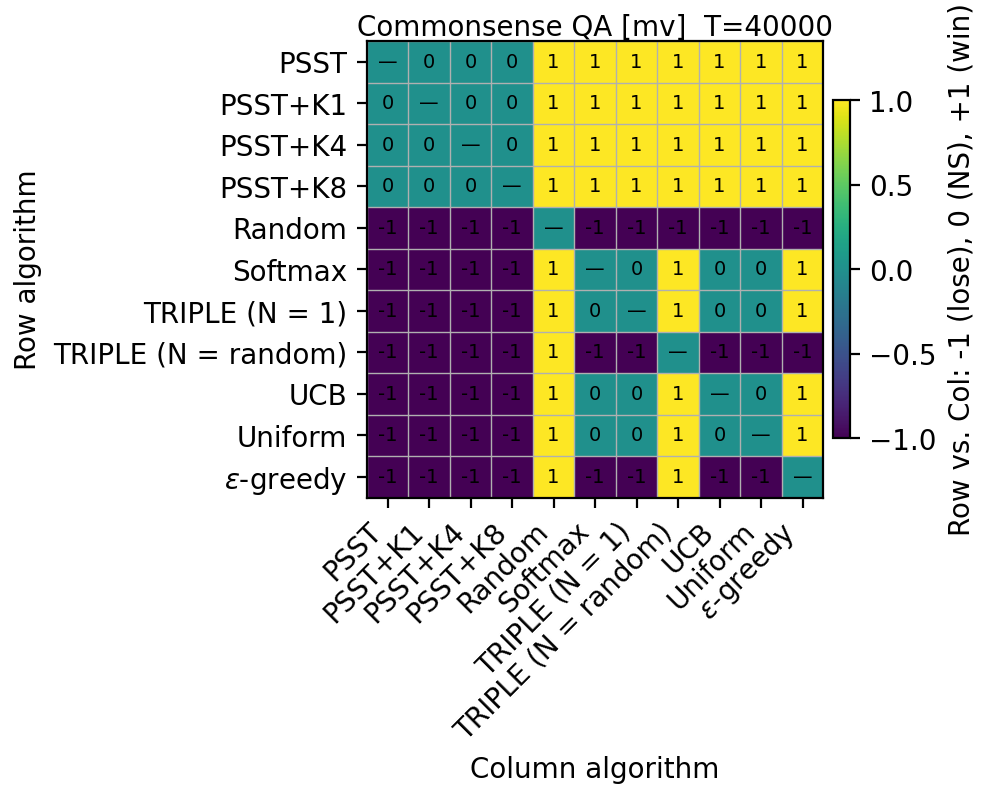}
\end{tabular}
\caption{Pairwise wins for Commonsense QA (MV) across six budgets ($T$ in order: 3000, 5000, 10000, 20000, 30000, 40000).}
\label{fig:cqa_mv_pairwise_grid}
\end{figure*}
\begin{figure*}[t]
\centering
\begin{tabular}{@{}ccc@{}}
\includegraphics[width=0.32\textwidth]{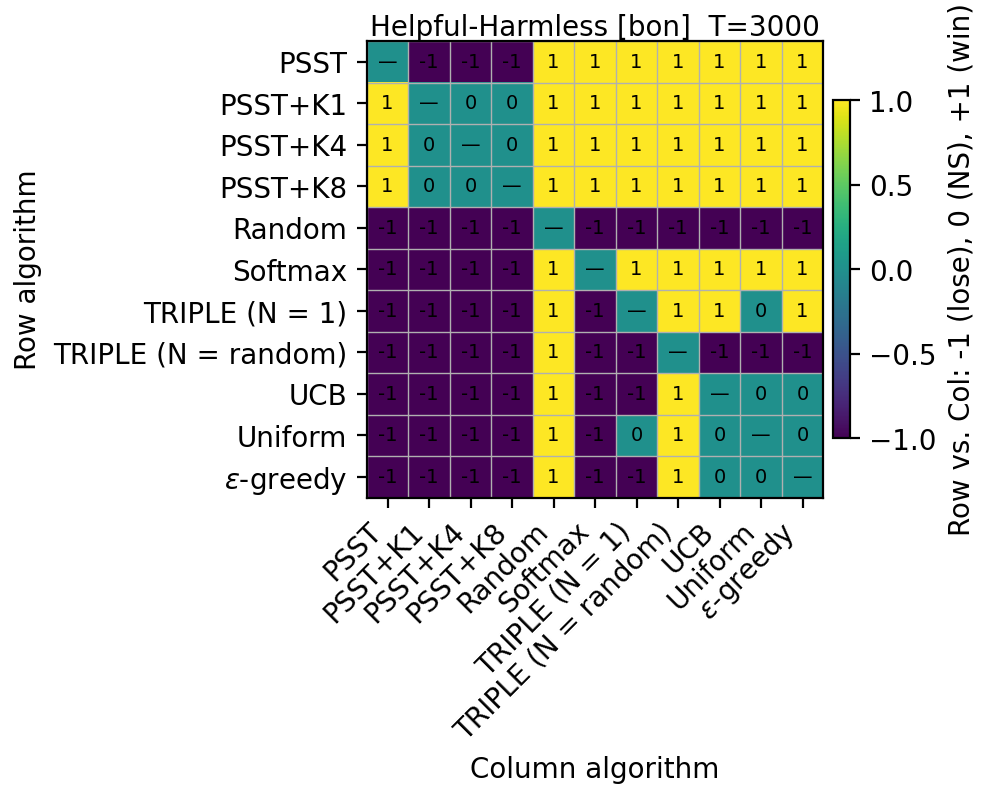} & \includegraphics[width=0.32\textwidth]{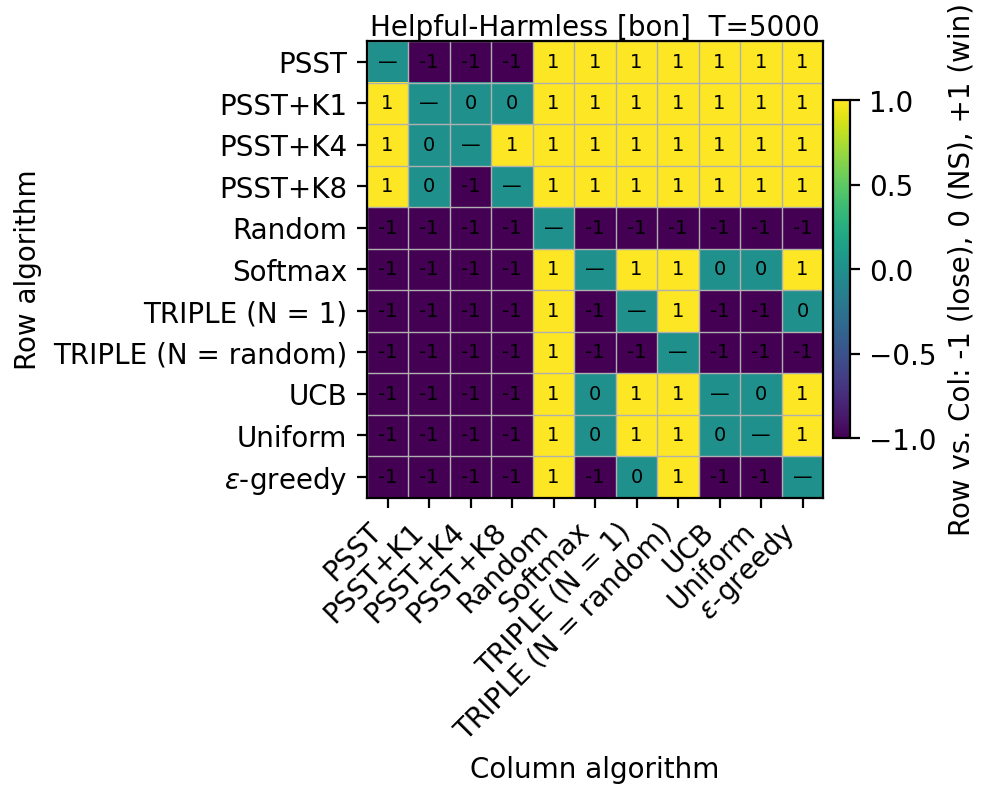} & \includegraphics[width=0.32\textwidth]{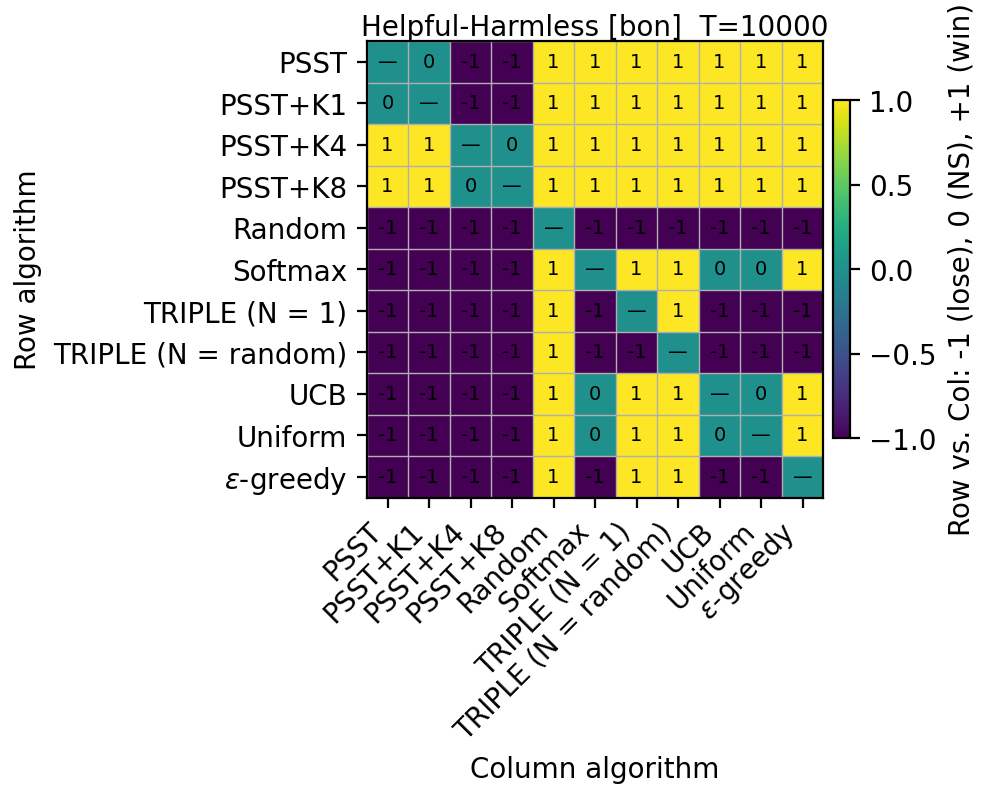} \\
[4pt]
\includegraphics[width=0.32\textwidth]{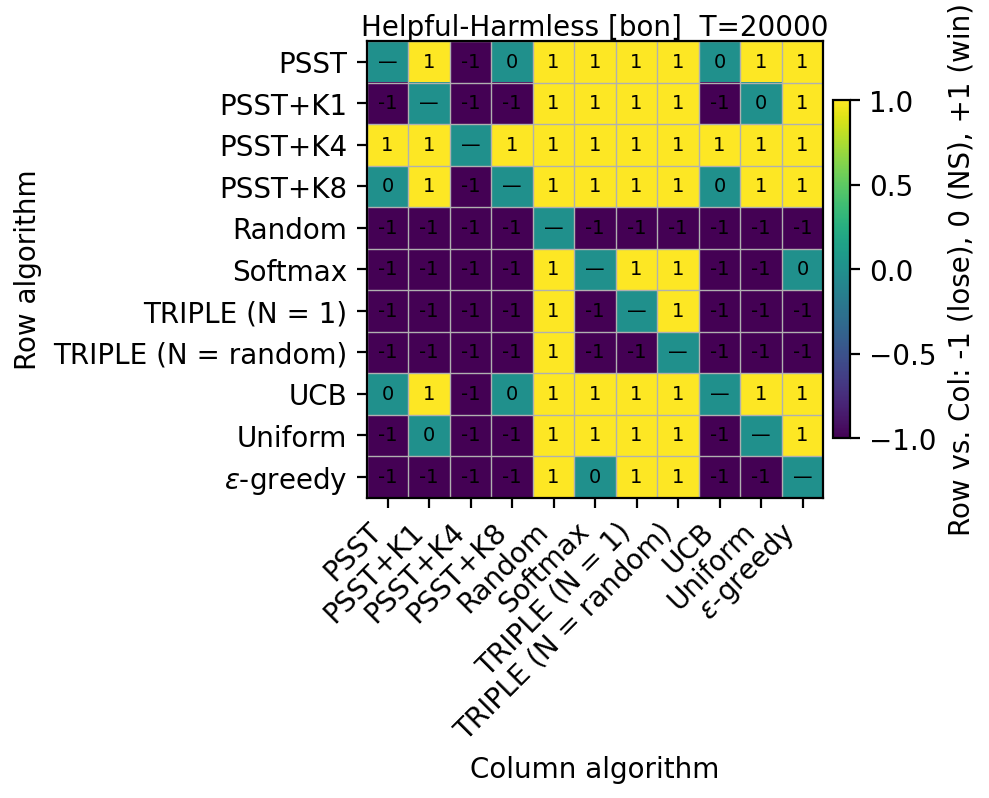} & \includegraphics[width=0.32\textwidth]{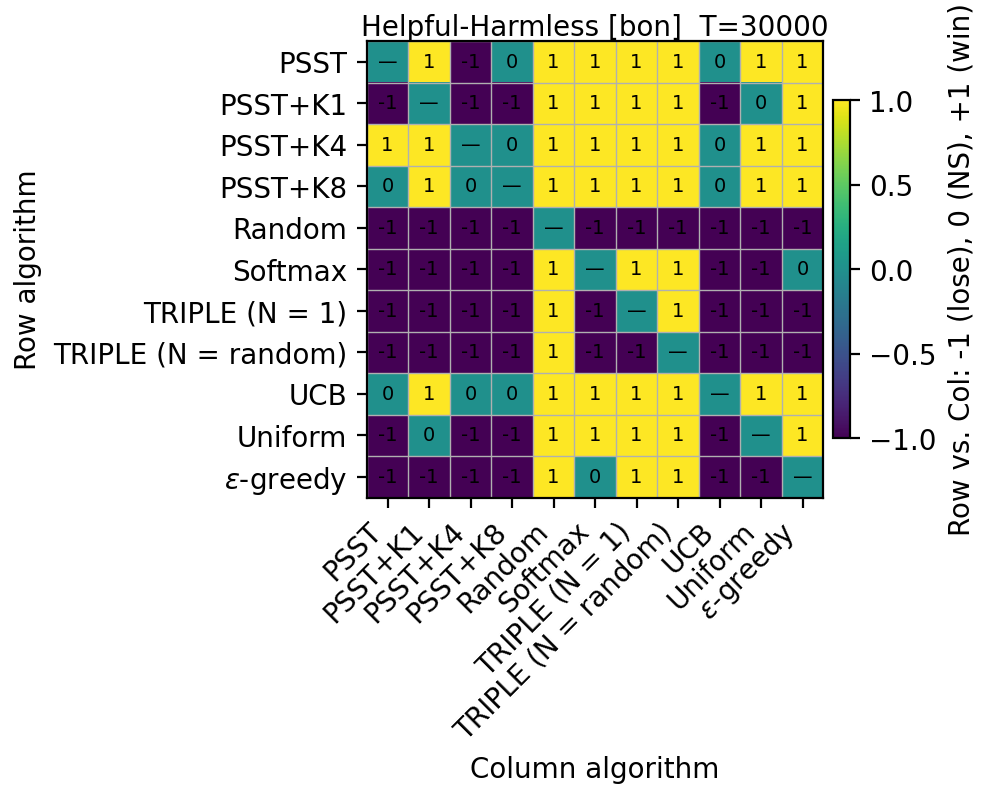} & \includegraphics[width=0.32\textwidth]{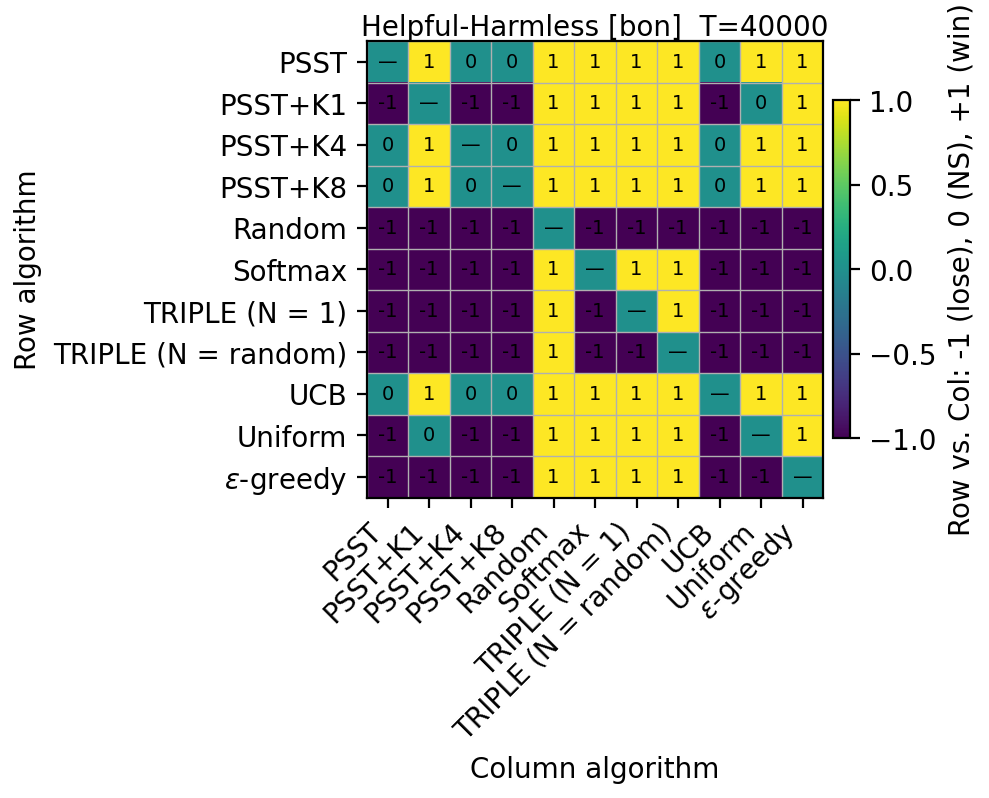}
\end{tabular}
\caption{Pairwise wins for Helpful-Harmless (BoN) across six budgets ($T$ in order: 3000, 5000, 10000, 20000, 30000, 40000).}
\label{fig:hh_bon_pairwise_grid}
\end{figure*}

\begin{figure*}[t]
\centering
\begin{tabular}{@{}ccc@{}}
\includegraphics[width=0.32\textwidth]{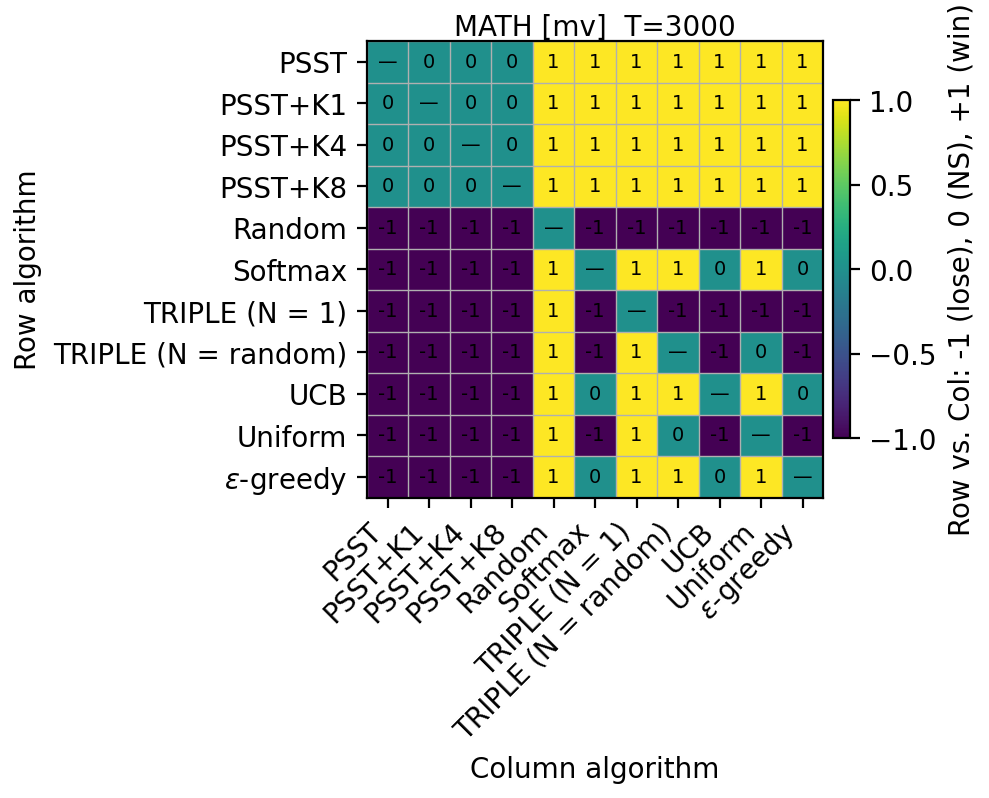} & \includegraphics[width=0.32\textwidth]{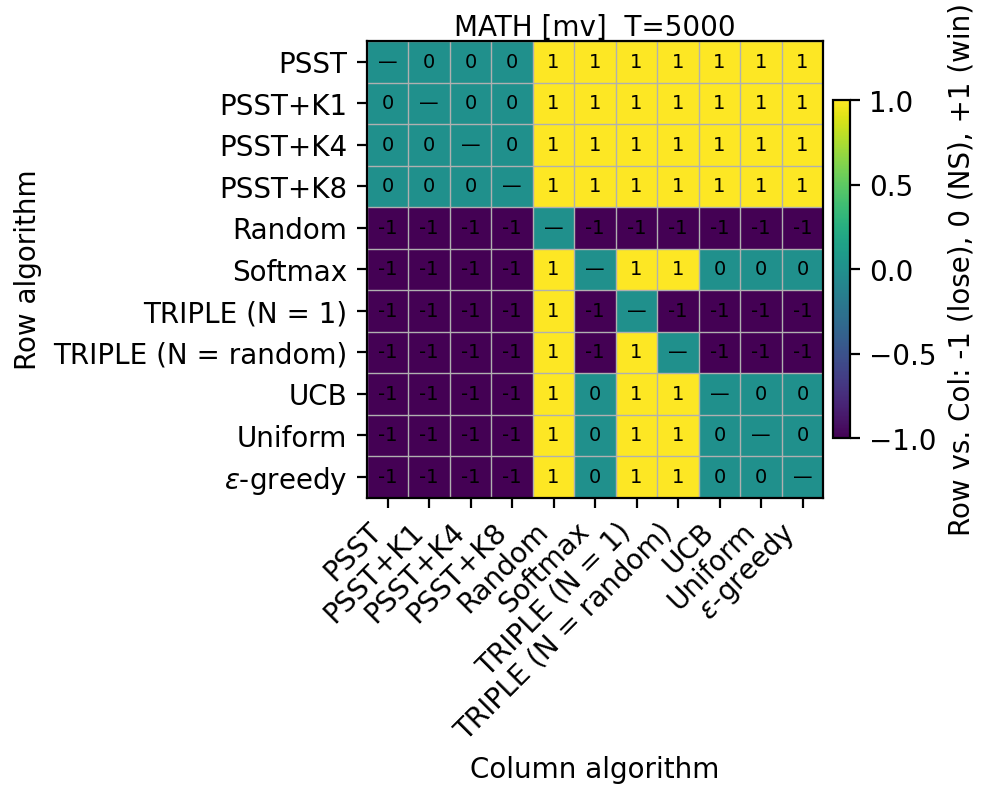} & \includegraphics[width=0.32\textwidth]{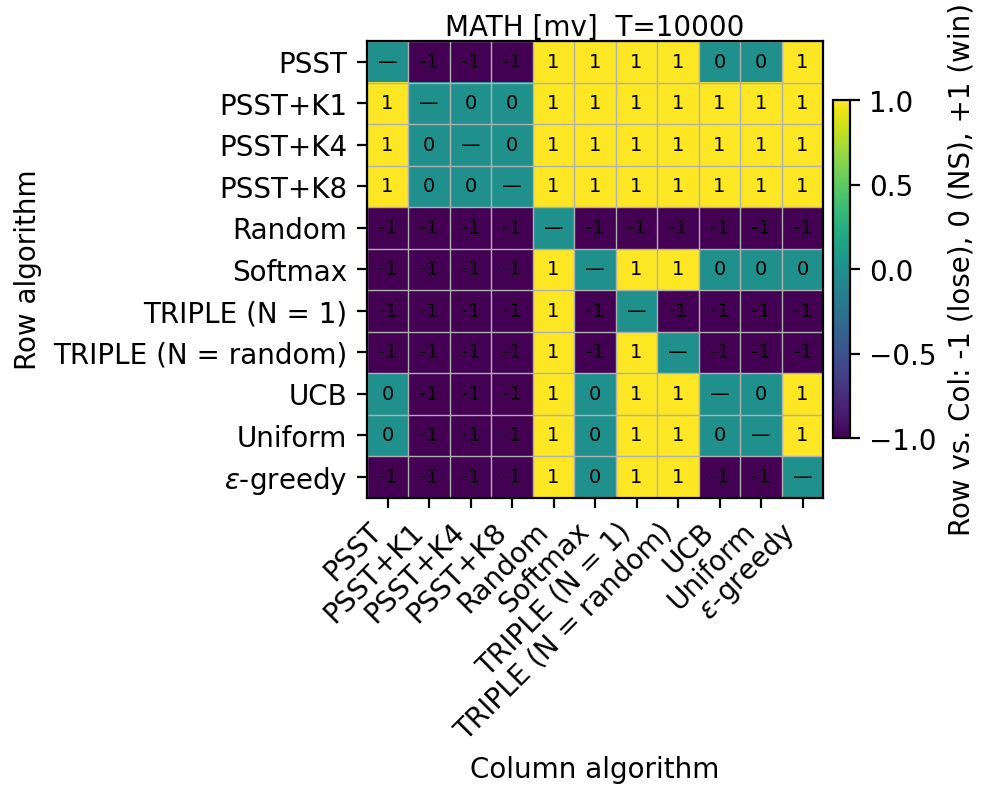} \\
[4pt]
\includegraphics[width=0.32\textwidth]{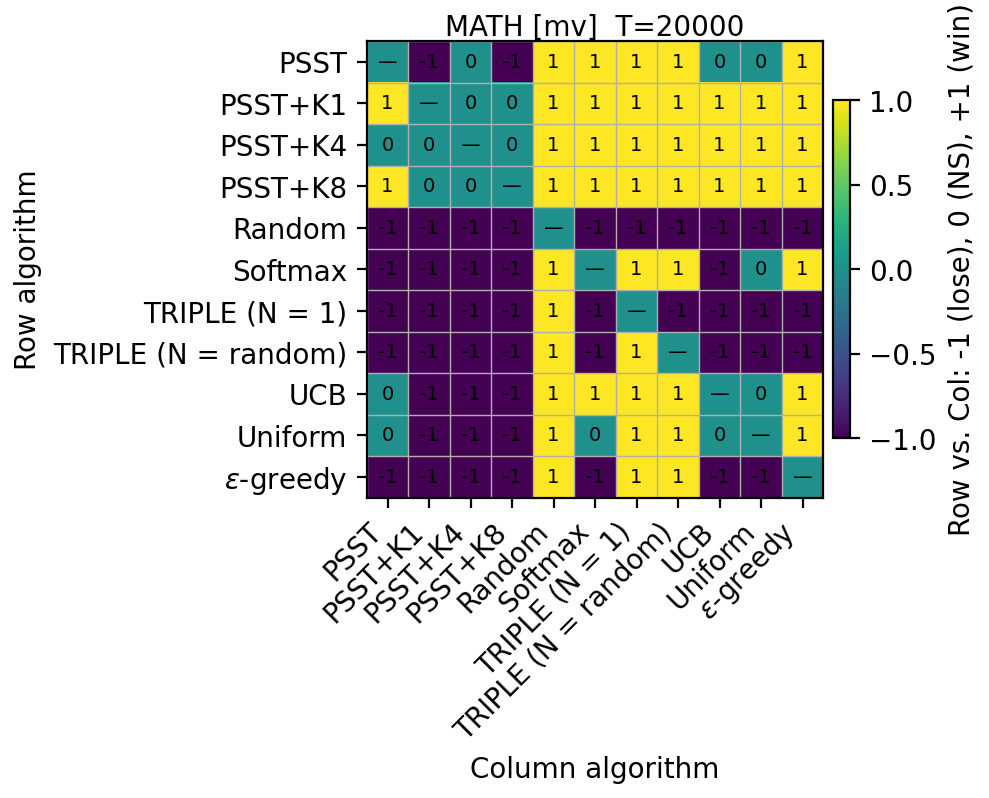} & \includegraphics[width=0.32\textwidth]{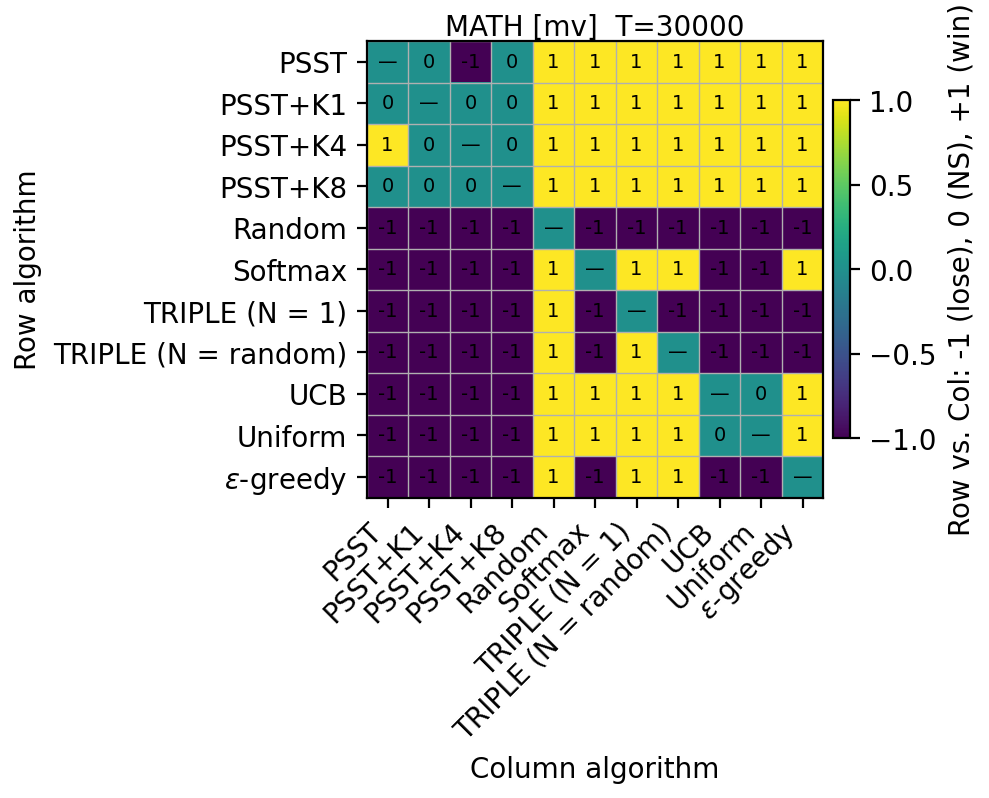} & \includegraphics[width=0.32\textwidth]{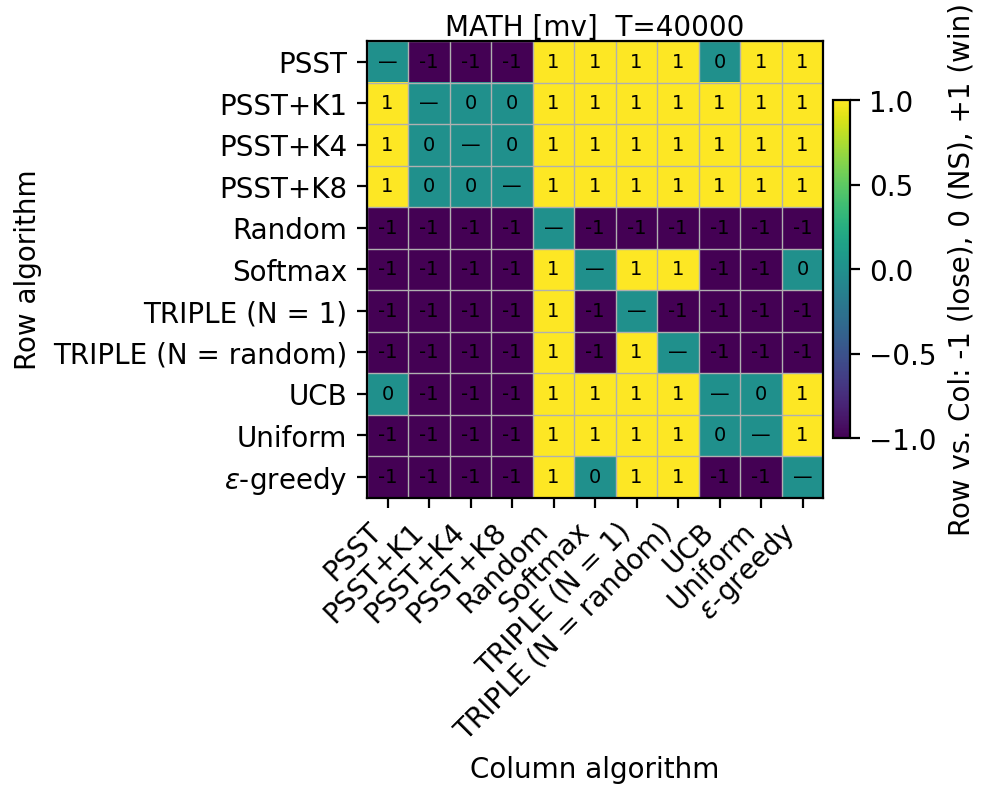}
\end{tabular}
\caption{Pairwise wins for MATH (MV) across six budgets ($T$ in order: 3000, 5000, 10000, 20000, 30000, 40000).}
\label{fig:math_mv_pairwise_grid}
\end{figure*}

\begin{figure*}[t]
\centering
\begin{tabular}{@{}ccc@{}}
\includegraphics[width=0.32\textwidth]{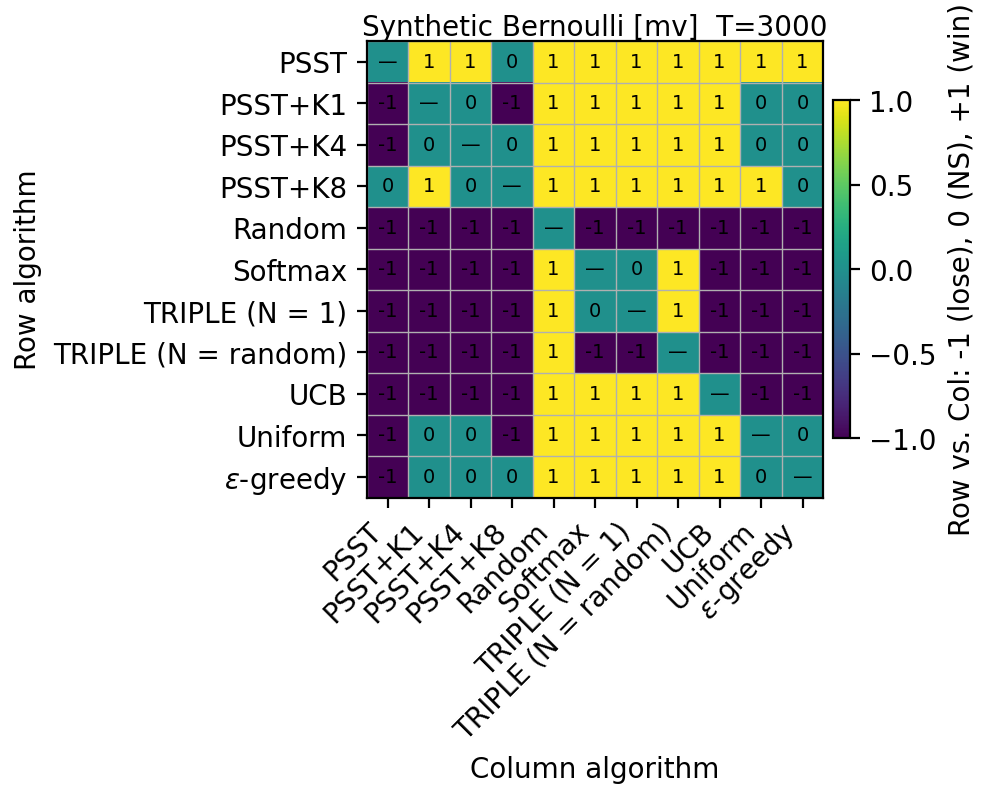} & \includegraphics[width=0.32\textwidth]{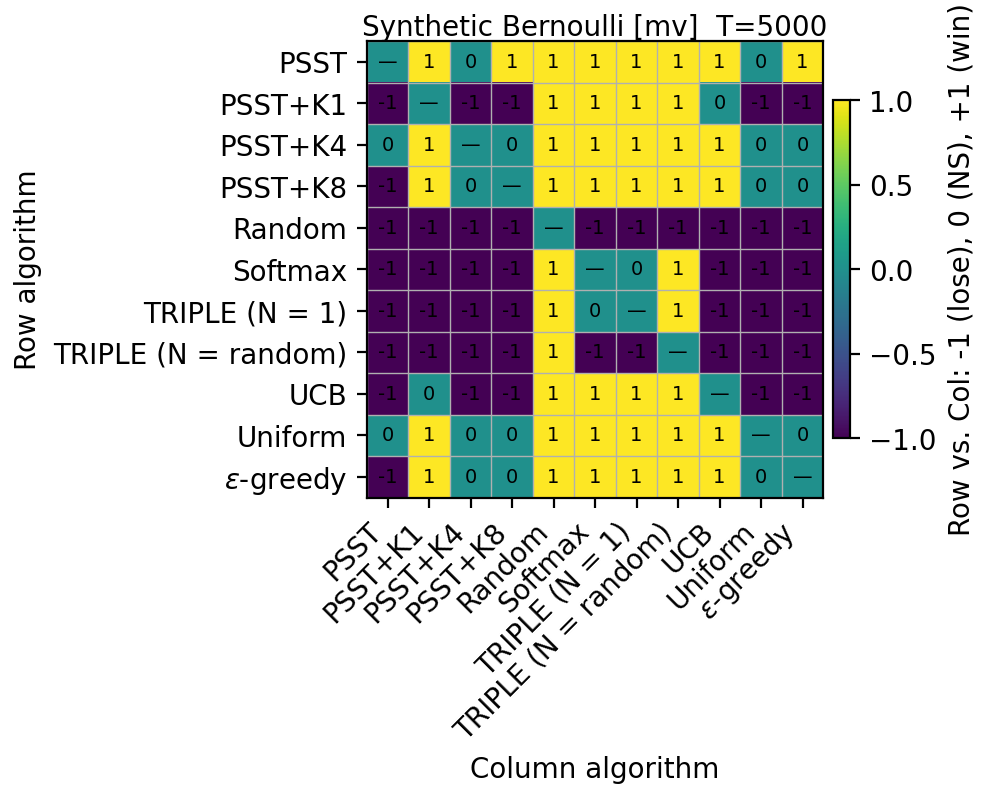} & \includegraphics[width=0.32\textwidth]{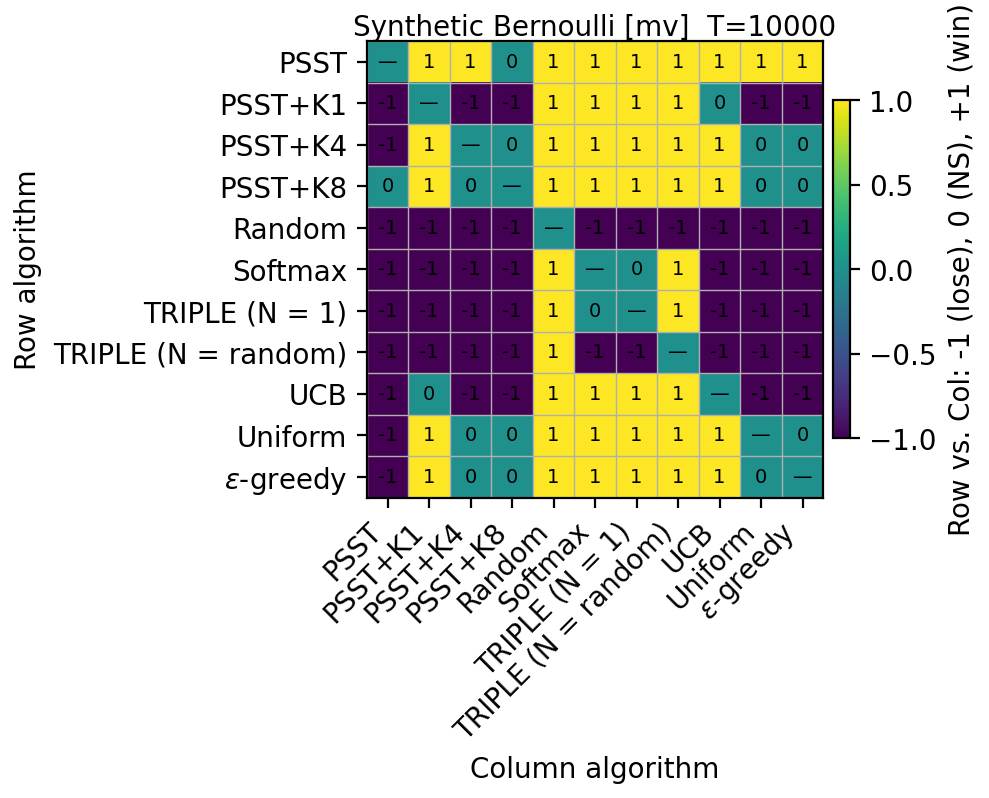} \\
[4pt]
\includegraphics[width=0.32\textwidth]{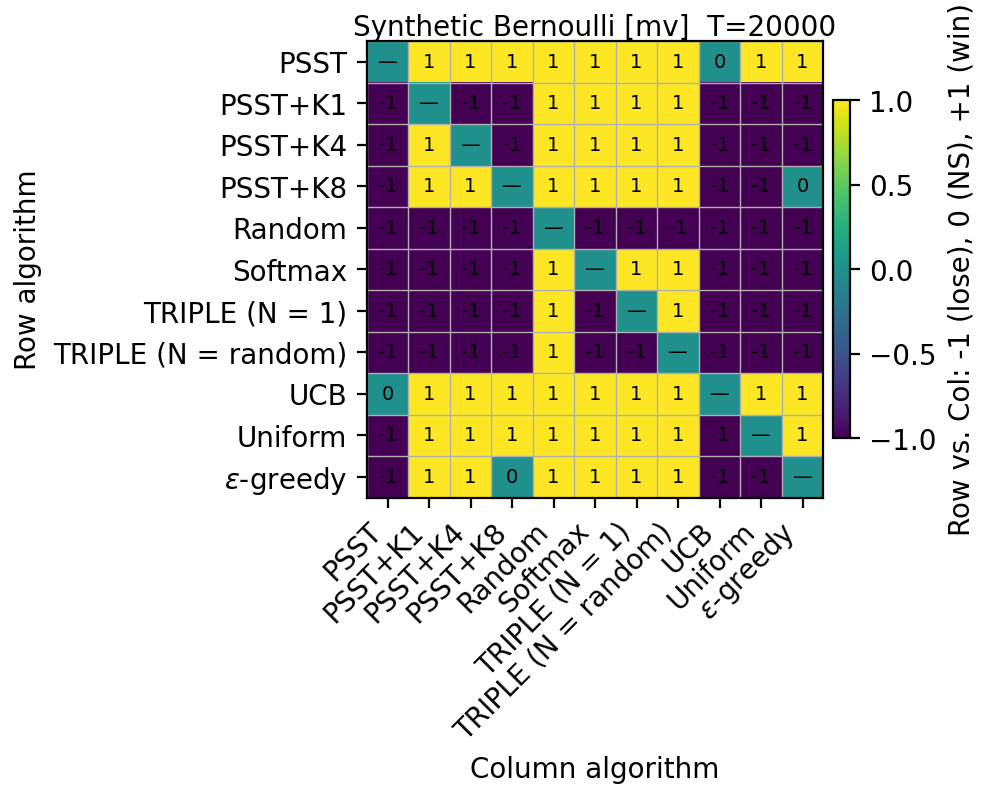} & \includegraphics[width=0.32\textwidth]{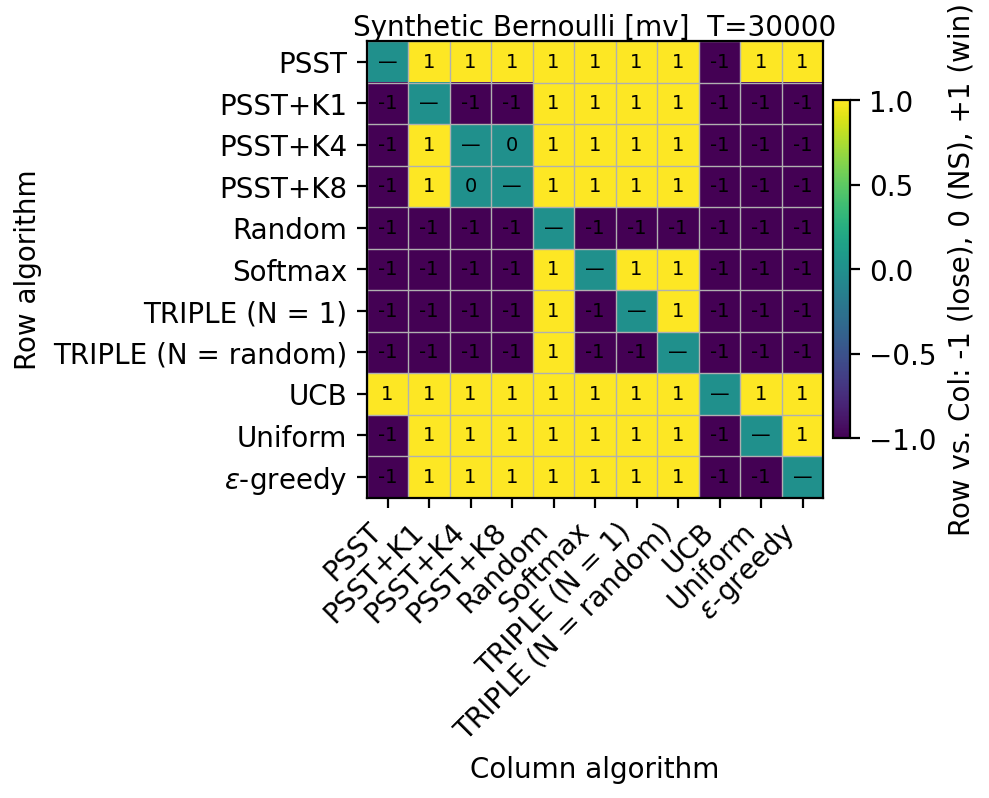} & \includegraphics[width=0.32\textwidth]{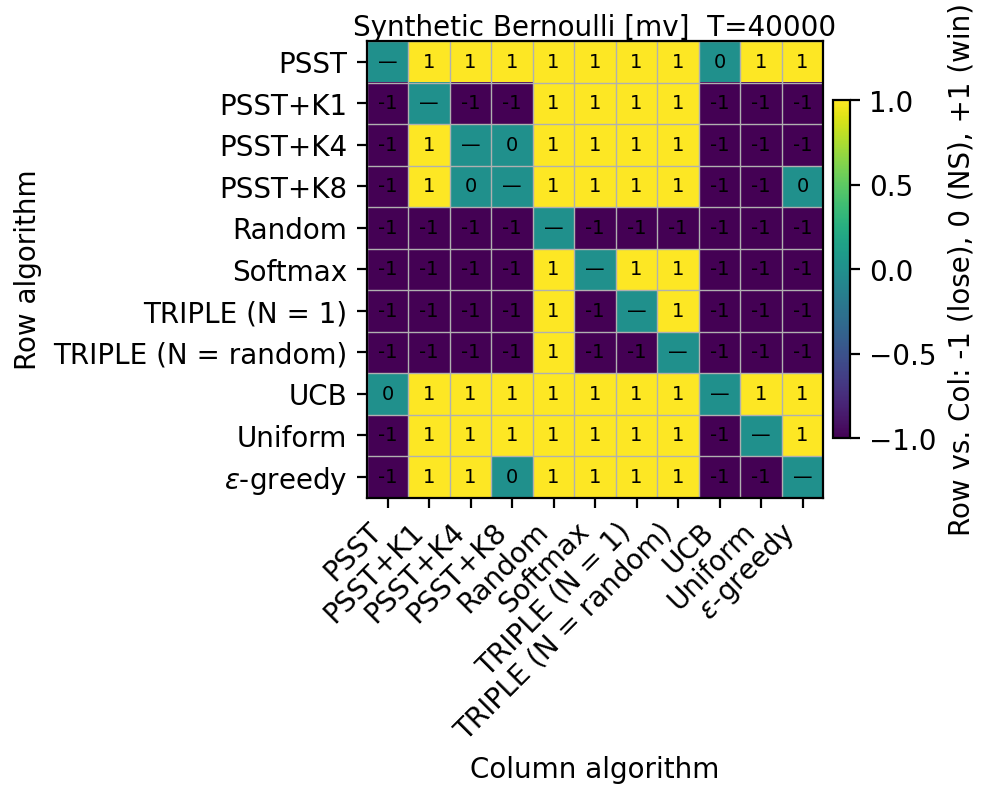}
\end{tabular}
\caption{Pairwise wins for Synthetic Bernoulli (MV) across six budgets ($T$ in order: 3000, 5000, 10000, 20000, 30000, 40000).}
\label{fig:sb_mv_pairwise_grid}
\end{figure*}

\begin{figure*}[t]
\centering
\begin{tabular}{@{}ccc@{}}
\includegraphics[width=0.32\textwidth]{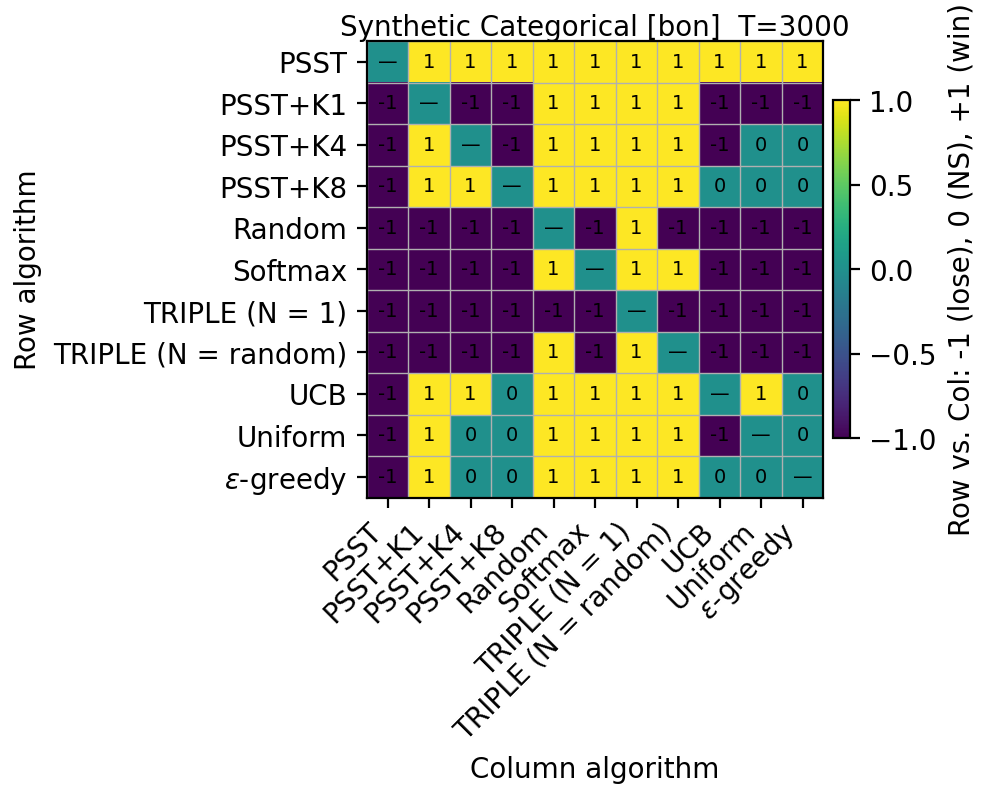} & \includegraphics[width=0.32\textwidth]{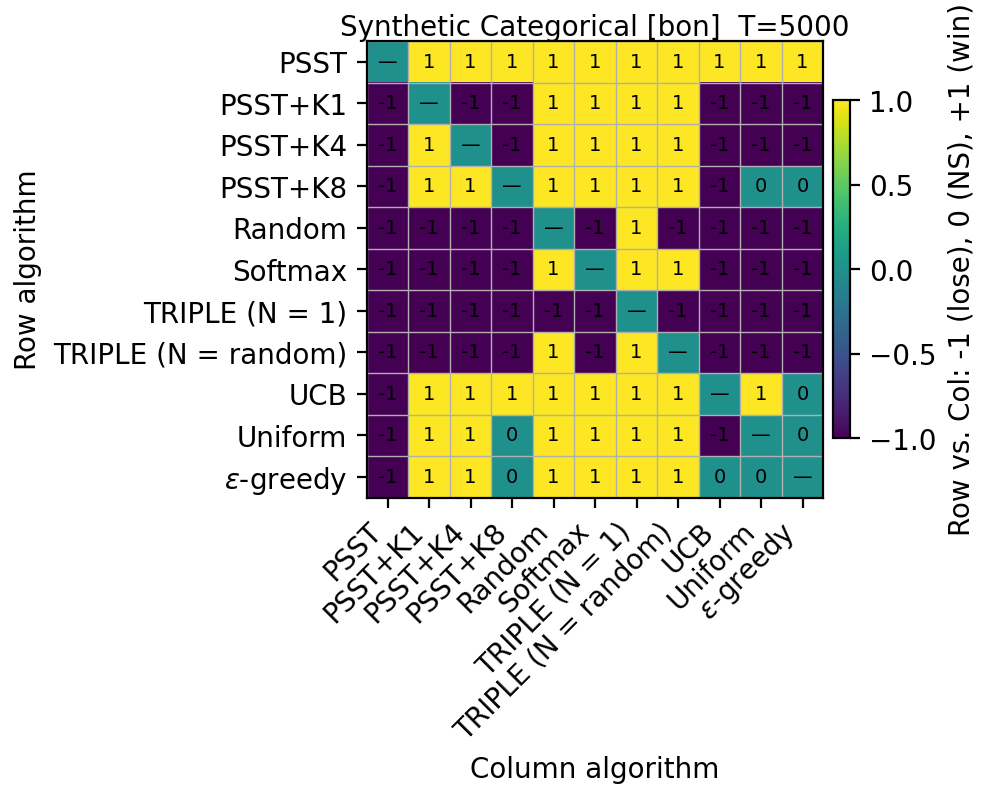} & \includegraphics[width=0.32\textwidth]{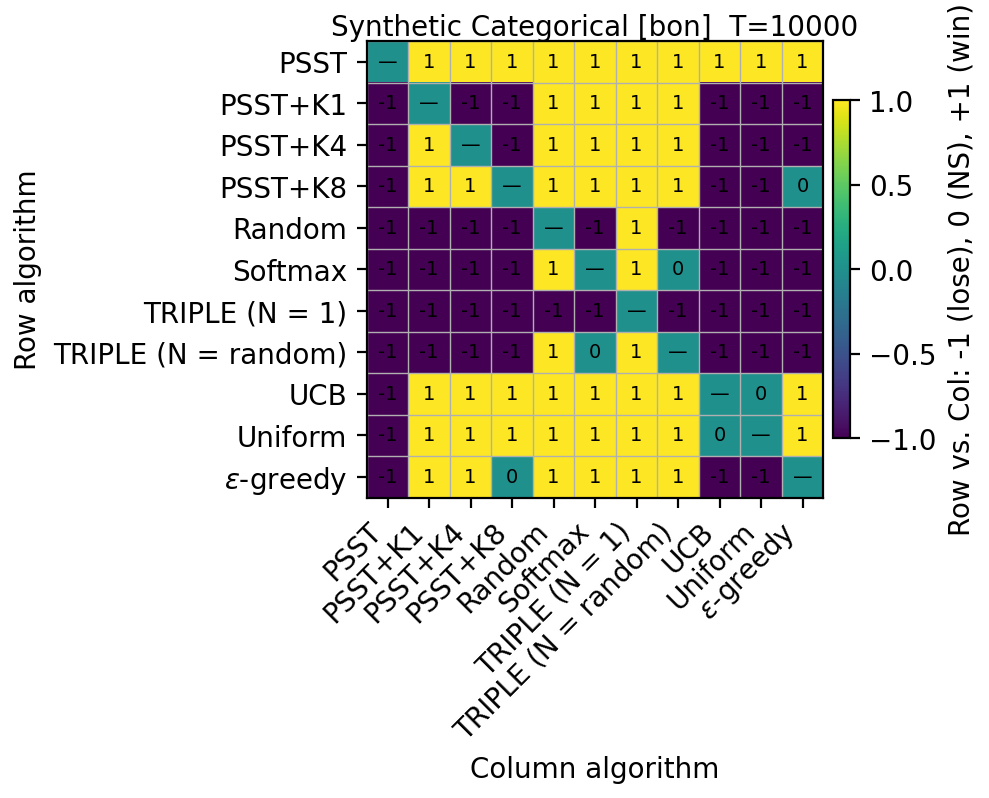} \\
[4pt]
\includegraphics[width=0.32\textwidth]{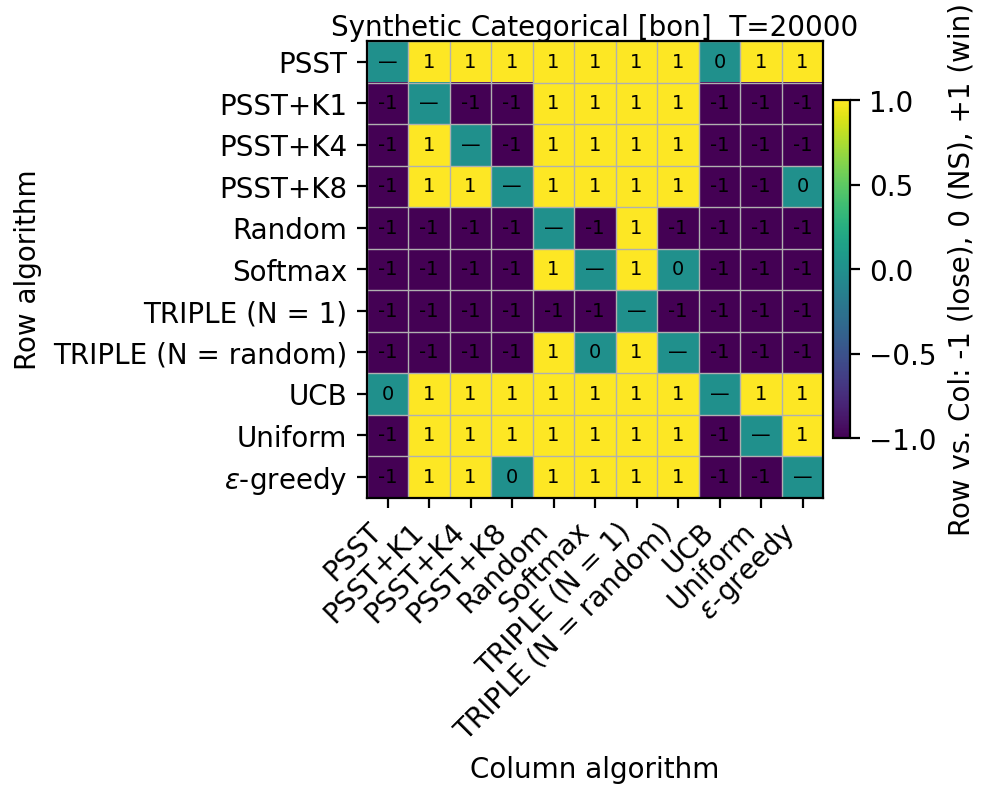} & \includegraphics[width=0.32\textwidth]{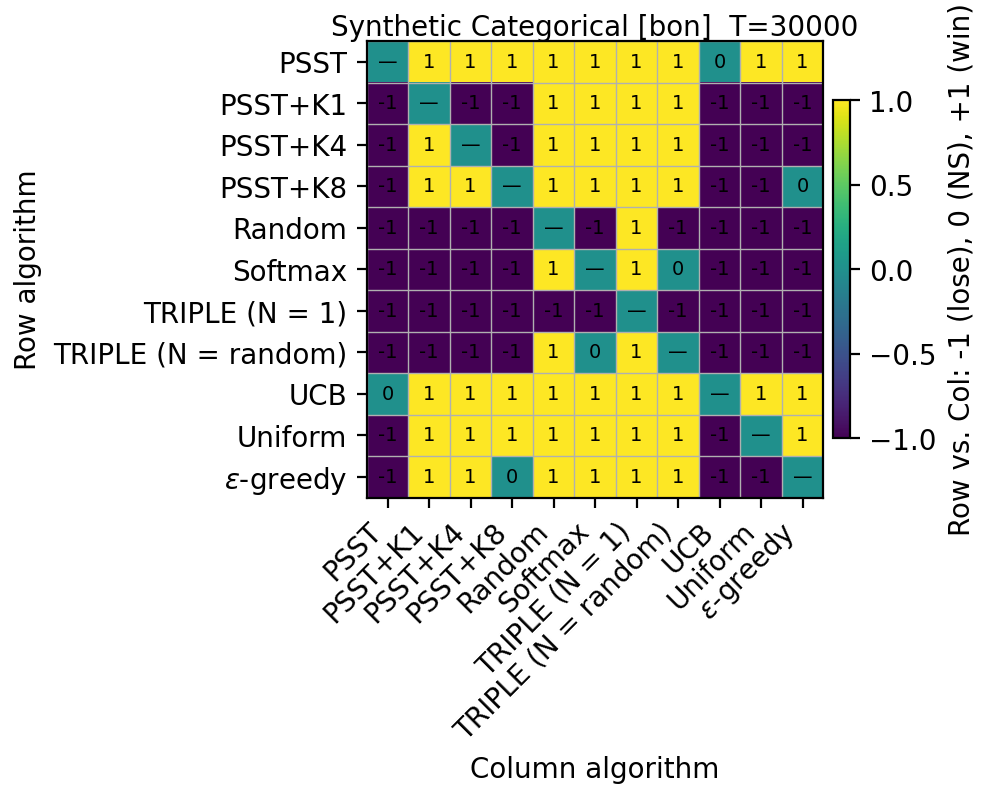} & \includegraphics[width=0.32\textwidth]{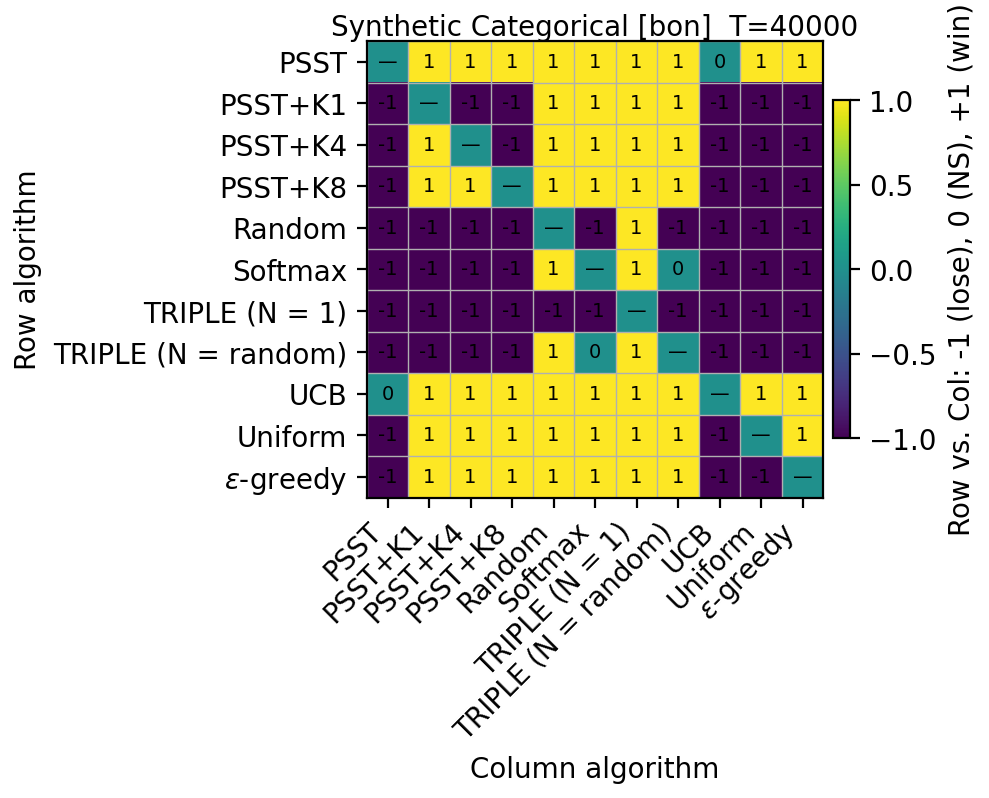}
\end{tabular}
\caption{Pairwise wins for Synthetic Categorical (BoN) across six budgets ($T$ in order: 3000, 5000, 10000, 20000, 30000, 40000).}
\label{fig:sc_bon_pairwise_grid}
\end{figure*}

\begin{figure*}[t]
\centering
\begin{tabular}{@{}ccc@{}}
\includegraphics[width=0.32\textwidth]{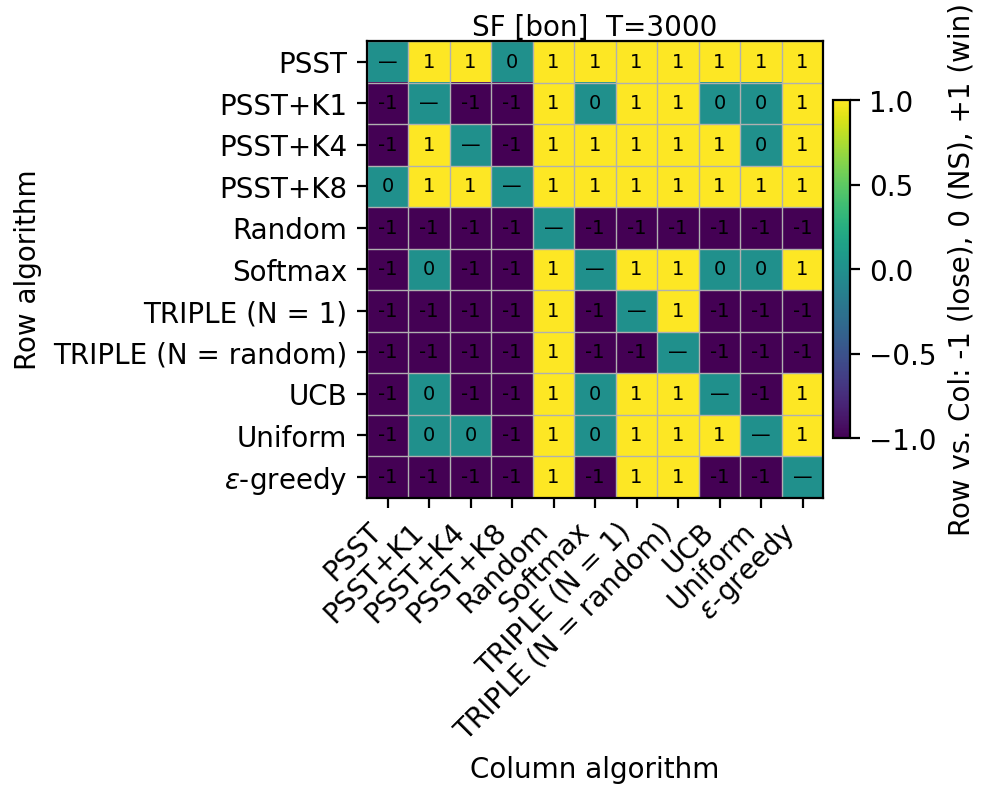} & \includegraphics[width=0.32\textwidth]{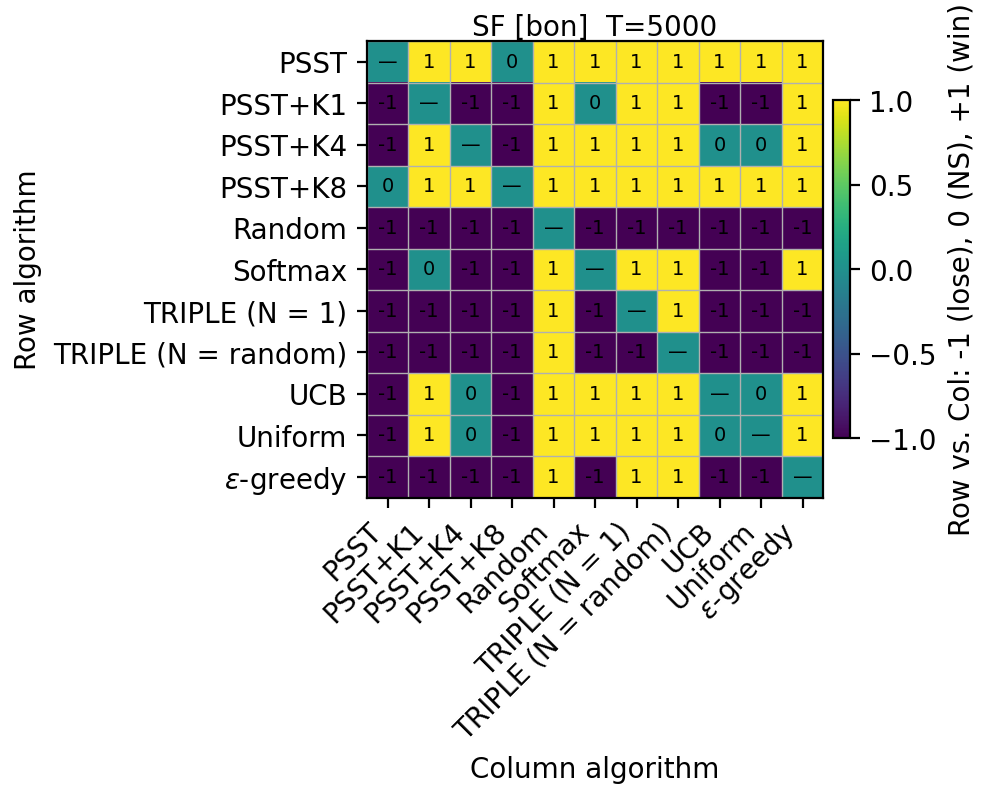} & \includegraphics[width=0.32\textwidth]{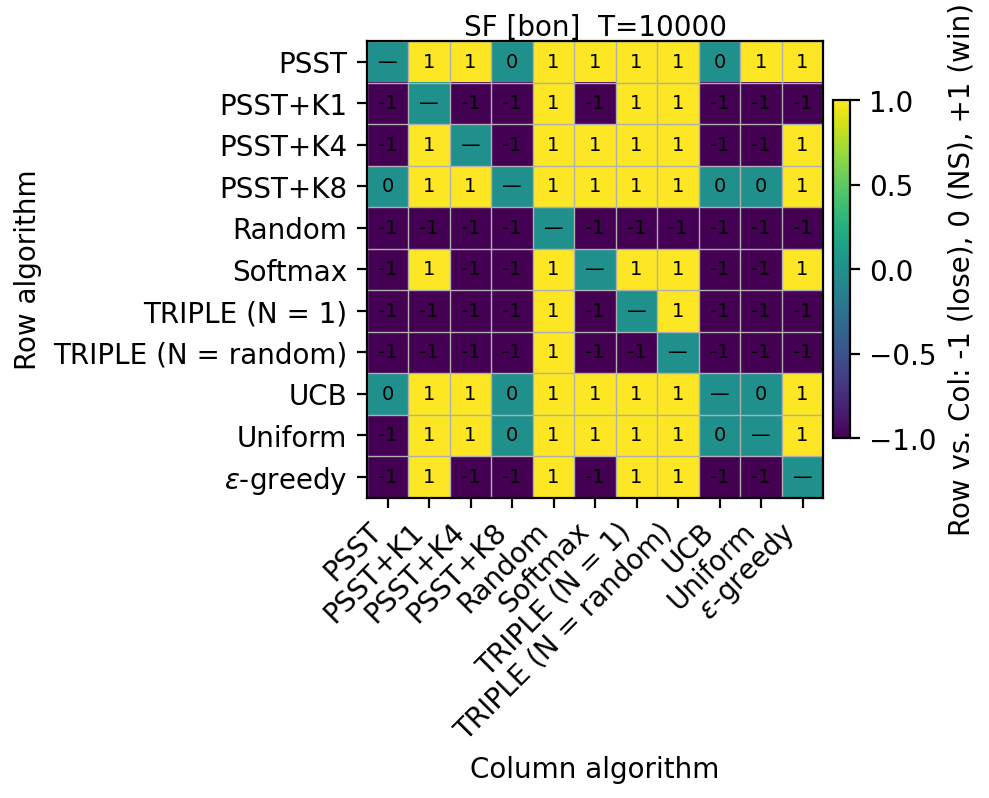} \\
[4pt]
\includegraphics[width=0.32\textwidth]{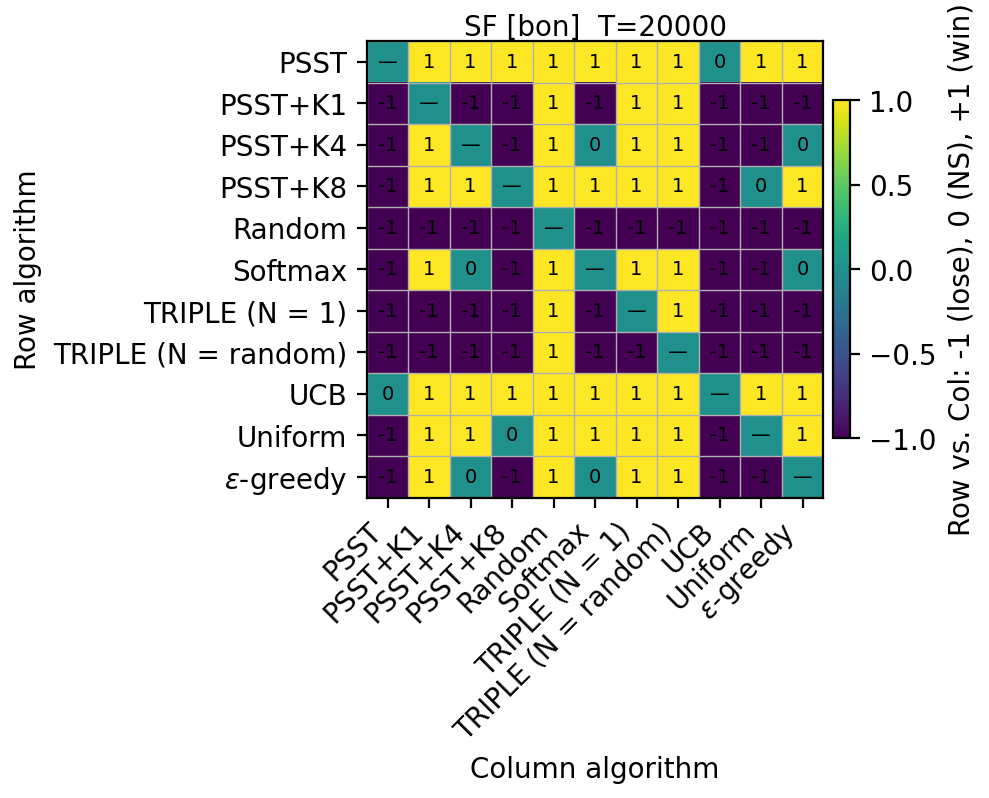} & \includegraphics[width=0.32\textwidth]{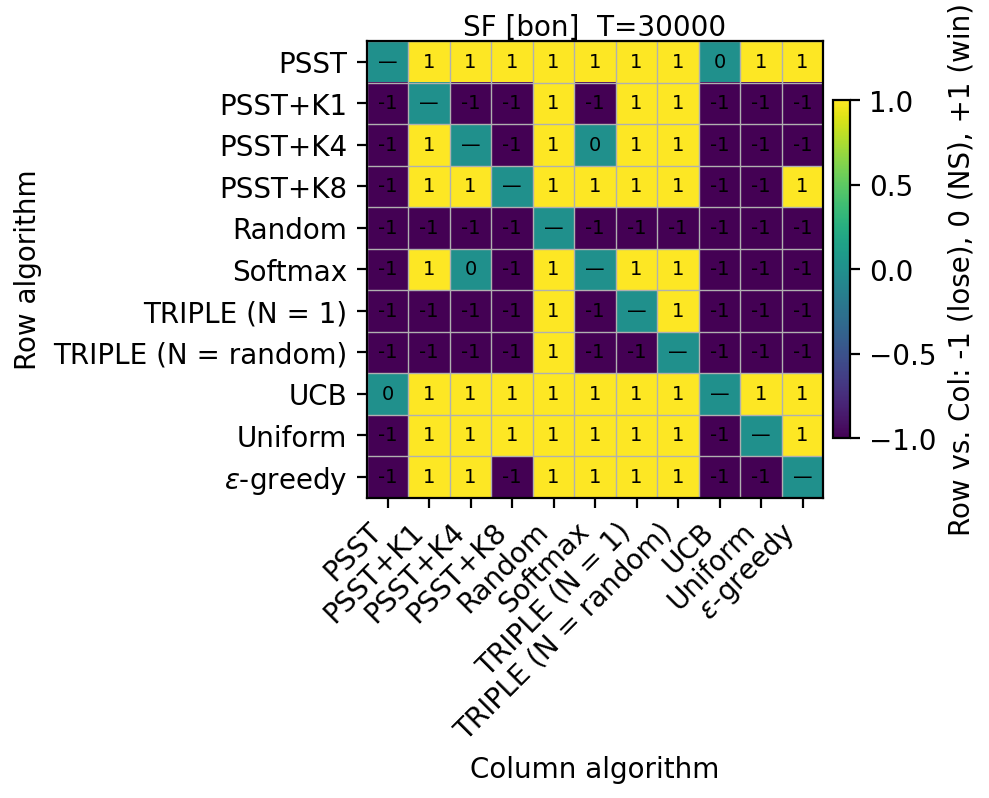} & \includegraphics[width=0.32\textwidth]{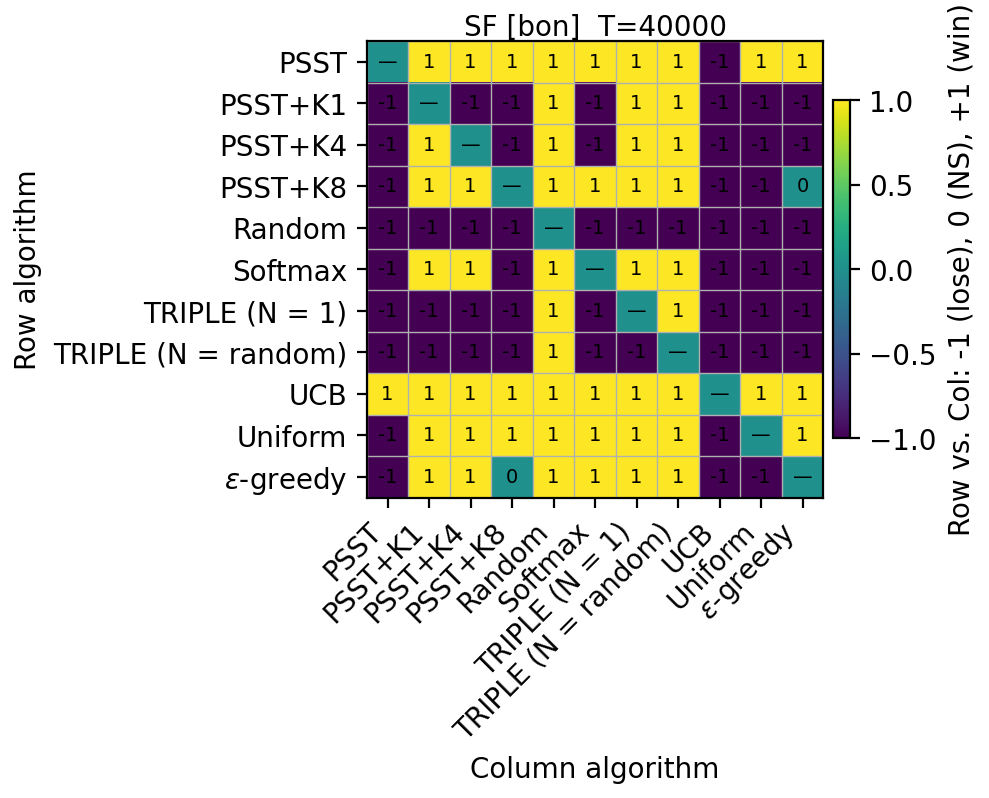}
\end{tabular}
\caption{Pairwise wins for Summarization (BoN) across six budgets ($T$ in order: 3000, 5000, 10000, 20000, 30000, 40000).}
\label{fig:summarization_bon_pairwise_grid}
\end{figure*}

\end{document}